\documentclass{article}

\usepackage[final,nonatbib]{nips_2018}
\usepackage[utf8]{inputenc} % allow utf-8 input
\usepackage[T1]{fontenc}    % use 8-bit T1 fonts
\usepackage{hyperref}       % hyperlinks
\usepackage{url}            % simple URL typesetting
\usepackage{booktabs}       % professional-quality tables
\usepackage{amsfonts}       % blackboard math symbols
\usepackage{nicefrac}       % compact symbols for 1/2, etc.
\usepackage{microtype}      % microtypography

\usepackage{amsmath}
\usepackage{amssymb}
\usepackage{graphicx}
\usepackage{url}
\usepackage[small,bf]{caption}
\usepackage{subcaption}
\usepackage{array}

\usepackage[ruled,vlined]{algorithm2e}
\usepackage[usenames, dvipsnames]{color}
\usepackage{csquotes}

\usepackage{mleftright}

\newtheorem{theorem}{Theorem}
\newtheorem{lemma}{Lemma}

\newtheorem{remark}{Remark}

\newcommand{\field}[1]{\mathbb{#1}}

\newcommand{\E}{\field{E}}

\newcommand{\argmin}{\mathop{\rm argmin}}

\newcommand{\sgn}{\mathrm{sgn}}

\newcommand{\scF}{\mathcal{F}}

\newcommand{\scB}{\mathcal{B}}

\newcommand{\CC}{\mathcal{C}}
\newcommand{\DD}{\mathcal{D}}
\newcommand{\dH}{{d_H}}
\newcommand{\dE}{{d_E}}

\newcommand{\scO}{\mathcal{O}}
\newcommand{\sctO}{\widetilde{\mathcal{O}}}

\newcommand{\scA}{\mathcal{A}}
\newcommand{\scT}{\mathcal{T}}
\newcommand{\kt}{\widetilde{K}(T,c^*)}

\newcommand{\AB}{\mathcal{AB}}
\newcommand{\LB}{\mathcal{LB}}
\newcommand{\hc}{\widehat{c}}

\newcommand{\err}{\mathrm{err}}

\newcommand{\hCC}{\widehat{\CC}}

\newcommand{\scL}{\mathcal{L}}

\newcommand{\CCC}{\field{C}}
\newcommand{\DDD}{\field{D}}

\renewcommand{\Pr}{\field{P}}

\newcommand{\scP}{\mathcal{P}}
\newcommand{\scS}{\mathcal{S}}

\usepackage{stackengine}
\newcommand{\defeq}{\overset{\mathrm{def}}{=\joinrel=}}
\newcommand{\ots}{{\sc{ots}}}
\newcommand{\wdp}{{\sc{wdp}}}
\newcommand{\nwdp}{{\sc{n-wdp}}}

\newcommand{\pa}{\mathrm{par}}
\newcommand{\lca}{\mathrm{lca}}
\newcommand{\de}{\mathrm{de}}

\newcommand{\lef}{\mathrm{left}}
\newcommand{\ri}{\mathrm{right}}

\newcommand\bnm[2]{\mleft(\knds\genfrac..{0pt}{}{#1}{#2}\knds\mright)}
\newcommand\comm[1]{\textcolor{ForestGreen}{\texttt{/*~#1~*/}}}
\newcommand{\knds}{\kern-\nulldelimiterspace}
\definecolor{verylightgray}{rgb}{0.8,0.8,0.8}

\newcommand{\hn}{\widetilde{H}}
\newcommand{\scD}{\mathcal{D}}
\newcommand{\kti}{\widetilde{K}}

\newenvironment{proof}
{\textit{Proof.}}
{$\square$}

\title{Flattening a Hierarchical Clustering through Active Learning}
\author{%
	   Fabio Vitale\\
	   Department of Computer Science\\
       INRIA Lille (France) \& \\ Sapienza University of Rome (Italy)\\
 	   \texttt{fabio.vitale@inria.fr}  \\
       \And
              Anand Rajagopalan\\
       Google (New York, USA)\\
       \texttt{arajagopalan@google.com}\\
       \And
       \And
		Claudio Gentile\\
       	Google (New York, USA)\\
       \texttt{cgentile@google.com}\\
}

\begin{document}

\maketitle

\begin{abstract}
\vspace{-0.1in}
%carry out a theoretical and experimental investigation 
We investigate active learning by pairwise similarity over the leaves of trees originating from hierarchical clustering procedures. In the realizable setting, we provide a full characterization of the number of queries needed to achieve perfect reconstruction of the tree cut. In the non-realizable setting, we rely on known important-sampling procedures to obtain regret and query complexity bounds. Our algorithms come with theoretical guarantees on the statistical error and, more importantly, lend themselves to {\em linear-time} implementations in the relevant parameters of the problem. We discuss such implementations, prove running time guarantees for them, and present preliminary experiments on real-world datasets showing the compelling practical performance of our algorithms as compared to both passive learning and simple active learning baselines.
\vspace{-0.1in}
\end{abstract}

%comparative

%originating from various linkage functions used in practice for hierarchical clustering. 
%These experiments show 

%thus paving the way for
%suggesting 
%deployability at scale of our algorithmic solutions.
% and, overall, the deployability at scale of these algorithms in real-world scenarios.

\vspace{-0.1in}
\section{Introduction}\label{s:intro}
\vspace{-0.12in}
Active learning is a learning scenario where %large amounts of unlabeled data are cheaply available, but
labeled data are scarse and/or %typically more 
expensive to gather, as they require careful assessment by human labelers. This is often the case in several practical settings where machine learning is routinely deployed, from image annotation to document classification, from speech recognition to spam detection, and beyond. In all such cases, an active learning algorithm tries to limit human intervention by seeking as little supervision as possible, still obtaining accurate prediction on unseen samples. This is an attractive learning framework offering substantial practical benefits, but also presenting statistical and algorithmic challenges.

\vspace{-0.05in}
A main argument that makes active learning effective is when combined with methods that exploit the {\em cluster structure} of data (e.g., \cite{hd08,kub15,c+19}, and references therein), where a cluster typically encodes some notion of semantic similarity across the involved data points. An obiquitous solution to clustering is to organize data into a {\em hierarchy}, delivering clustering solutions at different levels of resolution. An (agglomerative) Hierarchical Clustering (HC) procedure is an unsupervised learning method parametrized by a similarity function over the items to be clustered and a linkage function that lifts similarity from items to clusters of items. Finding the ``right'' level of resolution amounts to turning a given HC into a {\em flat} clustering by cutting the resulting tree appropriately. 
We would like to do so by resorting to human feedback in the form of 
%Now, when designing an active learning method, the first question we are facing is what kind of feedback the learning process is relying upon, that is, what kind of {\em queries} we are expecting the human will be able (or willing) to respond to. It is well known that 
{\em pairwise similarity} queries, that is, yes/no questions of the form ``are these two products similar to one another ?'' or ``are these two news items covering similar events ?''. It is well known that such queries are relatively easy to respond to, but are also intrinsically prone to subjectiveness and/or noise. More importantly, the hierarchy at hand need not be aligned with the similarity feedback we actually receive.

\vspace{-0.05in}
In this paper, we investigate the problem of cutting a tree originating from a pre-specified HC procedure through pairwise similarity queries generated by active learning algorithms. Since the tree is typically not consistent with the similarity feedback, that is to say, the feedback is {\em noisy},
%the inherent subjectivity of human judgement 
we are lead to tackle this problem under a variety of assumptions about the nature of this noise (from noiseless to random but persistent to general agnostic). Moreover, because different linkage functions applied to the very same set of items may give rise to widely different tree topologies, our study also focuses on characterizing active learning performance as a function of the structure of the tree at hand. Finally, because these hierarchies may in practice be sizeable (in the order of billion nodes), scalability will be a major concern in our investigation.

\vspace{-0.05in}
%Our contributions can be summarized as follows. 
\noindent{\bf Our contribution.} 
In the realizable setting (both noiseless and persistent noisy, Section \ref{s:realizable}), we introduce algorithms 
%(\wdp and \nwdp -- see Section \ref{s:realizable}) 
whose expected number of queries scale with the {\em average complexity} of tree cuts, a notion which is introduced in this paper. A distinctive feature of these algorithms is that they are rather ad hoc in the way they deal with the structure of our problem. In particular, they cannot be seen as finding the query that splits the version space as evenly as possible, a common approach in many active learning papers (e.g., \cite{da05,no11,gk17,gss13,td17,ml18}, and references therein). We then show that, at least in the noiseless case, this average complexity measure
%introduced in this paper, 
characterizes the expected query complexity of the problem. Our ad hoc analyses are beneficial in that they deliver sharper 
%query bound 
guarantees than those 
%obtained by directly applying the more general approaches analyzed in many of the 
readily available from the
above papers. In addition, and perhaps more importantly for practical usage, our algorithms admit {\em linear-time} implementations in the relevant parameters of the problem (like the number of items to be clustered).
In the non-realizable setting (Section \ref{s:nonrealizable}), we build on known results in importance-weighted active learning (e.g., \cite{bdl09,b+10}) to devise a selective sampling algorithm working under more general conditions. While our statistical analysis follows by adaptating available results,
% (specifically, those in \cite{b+10}), 
our goal here is to rather come up with fast implementations, so as to put the resulting algorithms on the same computational footing as those operating under (noisy) realizability assumptions. By leveraging the specific structure of our hypothesis space, we design a fast incremental algorithm for selective sampling whose running time per round is linear in the height of the tree. In turn, this effort paves the way for our experimental investigation (Section \ref{s:exp}), where we compare the effectiveness of the two above-mentioned approaches (realizable with persistent noise vs non-realizable) on real data originating from various linkage functions.
% used in practice in HC. 
Though quite preliminary in nature, these experiments seem to suggest that the algorithms originating from the persistent noise assumption exhibit more attractive learning curves than those working in the more general non-realizable setting.

\vspace{-0.05in}
\noindent{\bf Related work.}
The literature on active learning is vast, and we can hardly do it justice here. In what follows we confine ourselves to the references which we believe are closest to our paper.
Since our sample space is discrete (the set of all possible pairs of items from a finite set of size $n$), our realizable setting is essentially a {\em pool-based} active learning setting. Several papers have considered greedy algorithms which generalize binary search \cite{arkin+93,kosaraju+99,da05,no11,gk17,ml18}. The query complexity can be measured either in the worst case or averaged over a prior distribution over all possible labeling functions in a given set. The query complexity of these algorithms can be analyzed by comparing it to the best possible query complexity achieved for that set of items. In \cite{da05} it is shown that if the probability mass of the version space is split as evenly as possible then the approximation factor for its average query complexity is $\scO(\log(1/p_m))$, where $p_m$ is the minimal prior probability of any considered labeling function. \cite{gk17} extended this result through a more general approach to approximate greedy rules, but with the worse factor $\scO(\log^2(1/p_m))$. \cite{kosaraju+99} observed that modifying the prior distribution always allows one to replace $\scO(\log(1/p_m))$ by the smaller factor $\scO(\log N)$, where $N$ is the size of the set of labeling functions. Results of a similar flavor are contained in \cite{no11,ml18}. In our case, $N$ can be exponential in $n$ (see Section \ref{s:prel}), making these landmark results too broad to be tight for our specific setting. Furthermore, some of these papers (e.g., \cite{da05,no11,ml18}) have only theoretical interest because of their difficult algorithmic implementation. Interesting advances on this front are contained in the more recent paper \cite{td17}, though when adapted to our specific setting, their results give rise to worse query bounds than ours. In the same vein are the papers by \cite{chkk15,chk17}, dealing with persistent noise. 
Finally, in the non-realizable setting, our work fully relies on \cite{b+10}, which in turns builds on standard references like \cite{cal94,bdl09,ha07} -- see, e.g., the comprehensive survey by~\cite{ha14}.
Further references specifically related to clustering with queries are mentioned in Appendix \ref{as:comp}.

\vspace{-0.12in}
\section{Preliminaries and learning models}\label{s:prel}
\vspace{-0.12in}
We consider the problem of finding cuts of a given binary tree through pairwise similarity queries over its leaves. We are given in input a binary\footnote
{
In fact, the trees we can handle are more general than binary: we are making the binary assumption throughout for presentational convenience only.
}
tree $T$ originating from, say, an agglomerative (i.e., bottom-up) HC procedure (single linkage, complete linkage, etc.) applied to a set of items $L =\{x_1,\ldots,x_n\}$. Since $T$ is the result of successive (binary) merging operations from bottom to top, $T$ turns out to be a {\em strongly binary tree}\footnote
{
A strongly binary tree is a rooted binary tree for which the root is adjacent to either zero or two nodes, and all non-root nodes are adjacent to either one or three nodes.
}
and the items in $L$ are the leaves of $T$. We will denote by $V$ the set of nodes in $T$, including its leaves $L$, and by $r$ the root of $T$. The height of $T$ will be denoted by $h$. When referring to a subtree $T'$ of $T$, we will use the notation $V(T')$, $L(T')$, $r(T')$, and $h(T')$, respectively. 
%A generic node in $V$ (either internal or leaf node) will be denoted by $i$ or $v$. 
We also denote by $T(i)$ the subtree of $T$ rooted at node $i$, and by $L(i)$ the set of leaves of $T(i)$, so that $L(i) = L(T(i))$, and $r(T(i)) = i$. Moreover, $\pa(i)$ will denote the parent of node $i$ (in tree $T$), $\lef(i)$ will be the left-child of $i$, and $\ri(i)$ its right child.

\begin{figure}
\vspace{-0.0in}\hspace{-0.1in}
\begin{minipage}{0.65\textwidth}%%%
\vspace{-0.08in}
	%%%\begin{center}
		\begin{subfigure}{.51\textwidth}
		%	\centering
			\includegraphics[width=8.8cm]{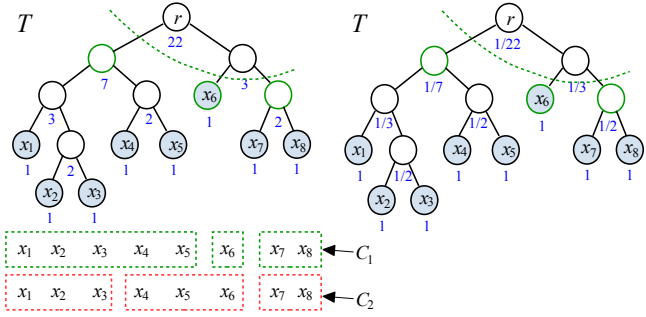}
		\end{subfigure}%	
	%%%\end{center}
	\end{minipage}%%%
%\vspace{-0.3in}
%\quad\quad\quad\quad\quad\quad\quad\quad\quad\quad\quad
%\hspace{-0.02em}
\begin{minipage}[l]{0.36\textwidth}
\vspace{-2.6in}
	\caption{{\bf Left:} A binary tree corresponding to a hierarchical clustering of the set of items $L = \{x_1,\ldots,x_8\}$. The cut depicted in dashed green has two nodes above and the rest below. This cut induces over $L$ the flat clustering $C_1 = \{\{x_1,x_2,x_3,x_4,x_5\},\{x_6\},\{x_7,x_8\}\}$ corresponding to the leaves of the subtrees rooted at the 3 green-bordered nodes just below the cut (the lower boundary of the cut). Clustering $C_1$ is therefore realized by $T$. On the contrary,\label{fig:tree}}
\end{minipage}%
\vspace{-2.71in}
{\small clustering $C_2 =\{\{x_1,x_2,x_3\},\{x_4,x_5,x_6\},\{x_7,x_8\}\}$ is not. Close to each node $i$ is also displayed the number $N(i)$ of realized cuts by the subtree rooted at $i$. For instance, in this figure, $7 = 1+3\cdot 2$, and $22 = 1+7\cdot 3$, so that $T$ admits overall $N(T) =22$ cuts. {\bf Right:} The same figure, where below each node $i$ are the probabilities $\Pr(i)$ encoding a uniform prior distribution over cuts. Notice that $\Pr(i) = 1/N(i)$ so that, like all other cuts, the depicted green cut has probability $(1-1/22)\cdot(1-1/3)\cdot(1/7)\cdot 1\cdot (1/2) = 1/22$}.
\vspace{-0.15in}
\end{figure}

\vspace{-0.05in}
A {\em flat} clustering $\CC$ of $L$ is a partition of $L$ into disjoint (and non-empty) subsets.
A {\em cut} $c$ of $T$ of size $K$ is a set of $K$ edges of $T$ that partitions $V$ into two disjoint subsets;
%of nodes. 
we call them
%these two subsets 
the nodes {\em above} $c$ and the nodes {\em below} $c$. Cut $c$ also univocally induces a clustering over $L$, made up of the clusters $L(i_1),L(i_2),\ldots, L(i_K)$, where $i_1,i_2, \ldots, i_K$ are the nodes below $c$ that the edges of $c$ are incident to. We denote this clustering by $\CC(c)$, and call the nodes $i_1,i_2, \ldots, i_K$ the {\em lower boundary} of $c$.
We say that clustering $\CC_0$ is {\em realized} by $T$ if there exists a cut $c$ of $T$ such that $\CC(c) = \CC_0$. See Figure \ref{fig:tree} (left) for a pictorial illustration. Clearly enough, for a given $L$, and a given tree $T$ with set of leaves $L$, not all possible clusterings over $L$ are realized by $T$, as the number and shape of the clusterings realized by $T$ are strongly influenced by $T$'s structure. Let $N(T)$ be the number of clusterings realized by $T$ (notice that this is also equal to the number of distinct cuts admitted by $T$). Then $N(T)$ can be computed through a simple recursive formula. If we let
%by $i_{\ell}$ and $i_r$ the left and right child of node $i$, respectively, and let 
$N(i)$ be the number of cuts realized by $T(i)$, one can easily verify that $N(i) = 1 + N(\lef(i))\cdot N(\ri(i))$, with $N(x_i) = 1$ for all $x_i \in L$. With this notation, we then have $N(T) = N(r(T))$. If $T$ has $n$ leaves, $N(T)$ ranges from $n$, when $T$ is a degenerate line tree, to the exponential $\lfloor \alpha^n\rfloor$, when $T$ is the full binary tree, where
%\footnote
%{
%The sequence of the number of cuts realized by a full binary tree as a function of its height is shown, e.g., in \url{http://oeis.org/A003095}.
%} 
$\alpha \simeq 1.502$ (e.g., \url{http://oeis.org/A003095}).
%= \exp\left(\sum_{j=0}^{\infty}2^{-j-1}\ln\left(1+a_j^{-2}\right)\right) \simeq 1.502$, being $a_i$ defined by the recurrence $a_i = a_{i-1}^2+1$, with $a_0=1$.
%Observe that $2^{2^{h_T-1}\le \lfloor c^n\rfloor<2^{2^{h_T}}$ for $h>1$.
See again Figure \ref{fig:tree} (left) for a simple example. 

\vspace{-0.05in}
A {\em ground-truth} matrix $\Sigma$ is an $n\times n$ and $\pm 1$-valued symmetric matrix $\Sigma =[\sigma(x_i,x_j)]_{i,j=1}^{n\times n}$ encoding a pairwise similarity relation over $L$. Specifically, if $\sigma(x_i,x_j)=1$ we say that $x_i$ and $x_j$ are similar, while if $\sigma(x_i,x_j)=-1$ we say they are dissimilar. Moreover, we always have $\sigma(x_i,x_i)=1$ for all $x_i \in L$. Notice that $\Sigma$ need not be consistent with a given clustering over $L$, i.e., 
%whereas 
the binary relation defined by $\Sigma$ over $L$
% is reflexive and symmetric, it 
need not be transitive.

\vspace{-0.05in}
Given $T$ and its leaves $L$, an active learning algorithm $A$ proceeds in a sequence of rounds. In a purely active setting, at round $t$, the algorithm queries a pair of items $(x_{i_t},x_{j_t})$, and observes the associated label $\sigma(x_{i_t},x_{j_t})$. In a {\em selective sampling} setting, at round $t$, the algorithm is presented with $(x_{i_t},x_{j_t})$ drawn from some distribution over $L\times L$, and has to decide whether or not to query the associated label $\sigma(x_{i_t},x_{j_t})$. In both cases, the algorithm is stopped at some point, and is compelled to commit to a specific cut of $T$ (inducing a flat clustering over $L$). Coarsely speaking, the goal of $A$ is to come up with a good cut of $T$, by making as few queries as possible on the entries of $\Sigma$. 
%In order to make the above more precise, we introduce a number of learning parameters that will help us define the specific settings we consider in this paper.
% the noise model on $\Sigma$, the error measure we are aimed at minimizing, and the amount of information available to the learning algorithm, encoded as a prior distribution over cuts of $T$. 
%This is made more precise next.

\vspace{-0.05in}
\noindent{\bf Noise Models.} 
The simplest possible setting, called {\em noiseless realizable} setting, is when $\Sigma$ itself is consistent with a given clustering realized by $T$, i.e., when there exists a cut $c^*$ of $T$ such that $\CC(c^*) = \{L(i_1), \ldots, L(i_K)\}$, for some nodes $i_1, \ldots, i_K \in V$, that satisfies the following: For all $r = 1, \ldots, K$, and for all pairs $(x_i,x_j) \in L(i_r) \times L(i_r)$ we have $\sigma(x_i,x_j) = 1$, while for all other pairs we have $\sigma(x_i,x_j) = -1$. We call (persistent) {\em noisy realizable} setting one where 
%the ground-truth 
$\Sigma$ is generated as follows. Start off from the noiseless ground-truth matrix,
% just defined for the noiseless setting, 
and call it $\Sigma^*$. Then, in order to obtain $\Sigma$ from $\Sigma^*$, consider the set of all $\binom{n}{2}$ pairs $(x_i,x_j)$ with $i < j$, and pick uniformly at random a subset of size $\lfloor\lambda\,\binom{n}{2}\rfloor$, for some $\lambda \in [0,1/2)$. Each such pair 
%$(x_i,x_j)$ in this subset 
has flipped label in $\Sigma$: $\sigma(x_i,x_j) = 1-\sigma^*(x_i,x_j)$. This is then combined with the symmetric $\sigma(x_i,x_j) = \sigma(x_j,x_i)$, and the reflexive $\sigma(x_i,x_i) = 1$ conditions. We call $\lambda$ the noise level. Notice that this kind of noise is random but {\em persistent}, in that if we query the same pair $(x_i,x_j)$ twice we do obtain the same answer $\sigma(x_i,x_j)$. Clearly, the special case $\lambda = 0$ corresponds to the noiseless setting. Finally, in the general {\em non-realizable} (or agnostic) setting, $\Sigma$ is an arbitrary matrix that need not be consistent with any clustering over $L$, in particular, with any clustering over $L$ realized by $T$.

\vspace{-0.05in}
\noindent{\bf Error Measure.} If $\Sigma$ is some ground-truth matrix over $L$, and $\hc$ is the cut output by $A$, with induced clustering $\hCC = \CC(\hc)$,
%$ = \{{\hat C}_1,{\hat C}_2,\ldots \}$,
% there are several natural ways to compute the distance between them (e.g., \cite{me07}). In this paper, we consider the following one:
%
we let $\Sigma_{\hCC} =[\sigma_{\hCC}(x_i,x_j)]_{i,j=1}^{n\times n}$ be the similarity matrix associated with $\hCC$, i.e., $\sigma_{\hCC}(x_i,x_j) = 1$ if $x_i$ and $x_j$ belong to the same cluster, and $-1$ otherwise. Then the {\em Hamming} distance $\dH(\Sigma,\hCC)$ 
%between matrix $\Sigma$ and clustering $\hCC$ 
%is a simple pairwise distance that 
simply counts the number of pairs $(x_i,x_j)$ having inconsistent sign:
\(
\dH(\Sigma,\hCC) = \left|\{(x_i,x_j) \in L^2\,:\, \sigma(x_i,x_j)\neq\sigma_{\hCC}(x_i,x_j) \} \right|\,.
\)
The same definition applies in particular to the case when $\Sigma$ itself represents a clustering over $L$. The quantity $\dH$, sometimes called correlation clustering distance, is %very 
closely related to the
%the so-called Mirkin metric \citep{mir96} over clusterings, as well as to the 
Rand index \cite{ra71} -- see, e.g., \cite{me11}.

\vspace{-0.05in}
\noindent{\bf Prior distribution.} Recall cut $c^*$ defined in the noiseless realizable setting and its associated 
%ground-truth matrix 
$\Sigma^*$. Depending on the specific learning model we consider (see below), the algorithm may have access to a {\em prior} distribution $\Pr(\cdot)$ over $c^*$, parametrized as follows. For $i \in V$, let $\Pr(i)$ be the conditional probability that $i$ is below $c^*$ given that all $i$'s ancestors are above.
%ancestors of $i$ in $T$ are above $c^*$. 
If we denote by $\AB(c^*) \subseteq V $ the nodes of $T$ which are above $c^*$, and by $\LB(c^*) \subseteq V$ those on the lower boundary of $c^*$, we can write
\vspace{-0.05in}
\begin{equation}\label{e:prob}
\Pr(c^*)=\Bigl(\prod_{i\in \AB(c^*)} (1-\Pr(i))\Bigl)\cdot\Bigl(\prod_{j\in \LB(c^*)} \Pr(j)\Bigl)~,
\end{equation}
where $\Pr(i) = 1$ if $i \in L$. In particular, setting $\Pr(i) = 1/N(i)$ $\forall i$ yields the {\em uniform} prior $\Pr(c^*) = 1/N(T)$ for all $c^*$ realized by $T$. See Figure \ref{fig:tree} (right) for an illustration. A canonical example of a non-uniform prior is one that favors cuts close to the root, which are thereby inducing clusterings having few clusters. These can be obtained, e.g., by setting
% $\Pr(i)$ in (\ref{e:prob}) as 
$\Pr(i) = \alpha$, for some constant $\alpha \in (0,1)$.
%where $h(i)$ is the height (or depth) of $i$ in $T$, and $f(\cdot)$ is some monotonically decreasing function of its argument, e.g., $f(x) = 2^{-x}$.

\vspace{-0.05in}
\noindent{\bf Learning models.} We consider two learning settings.
%With the above definitions handy, we can define the two learning settings we consider in this paper, executed in Section \ref{s:realizable} and Section \ref{s:nonrealizable} in turn.
%
The first setting (Section \ref{s:realizable}) is an active learning setting under a noisy realizability assumption with prior information. Let $\CC^* = \CC(c^*)$ be the ground truth clustering induced by cut $c^*$ before noise is added. Here, for a given prior $\Pr(c^*)$, the goal of learning is to identify $\CC^*$ either exactly (when $\lambda = 0$) or approximately (when $\lambda > 0$), while bounding the expected number of queries $(x_{i_t},x_{j_t})$ made to the ground-truth matrix $\Sigma$, the expectation being over the noise, and possibly over $\Pr(c^*)$. In particular, if $\hCC$ is the clustering produced by the algorithm after it stops, we would like to prove upper bounds on $\E[\dH(\Sigma^*,\hCC)]$,
% or on $\E[\dE(\CC^*,\hCC)]$, 
as related to the number of active learning rounds, as well as to the properties of the prior distribution.
The second setting (Section \ref{s:nonrealizable}) is a selective sampling setting where the pairs $(x_{i_t},x_{j_t})$ are drawn i.i.d. according to an arbitrary and unknown distribution $\DD$ over the $n^2$ entries of $\Sigma$, and the algorithm at every round can choose whether or not to query the label. After a given number of rounds the algorithm is stopped, and the goal is the typical goal of agnostic learning: no prior distribution over cuts is available anymore, and we would like to bound with high probability over the sample $(x_{i_1},x_{j_1}), (x_{i_2},x_{j_2}), \ldots $ the so-called {\em excess risk} of the clustering $\hCC$ produced by $A$, i.e., the difference
\begin{equation}\label{e:excessrisk}
\Pr_{(x_i,x_j) \sim \DD}\left( \sigma(x_i,x_j) \neq \sigma_{\hCC}(x_i,x_j) \right) - \min_{c} \Pr_{(x_i,x_j) \sim \DD}\left( \sigma(x_i,x_j) \neq \sigma_{\CC(c)}(x_i,x_j) \right)~,
\end{equation}
the minimum being over all possible cuts $c$ realized by $T$. Notice that when $\DD$ is uniform 
%over $L^2$ 
the excess risk reduces to 
\(
\frac{1}{n^2}\,\left(\dH(\Sigma,\hCC) - \min_{c} \dH(\Sigma,\CC(c))\right).
\)
At the same time, we would like to bound with high probability
% (again, over the sample)
the total number of labels the algorithm has queried.

%When $\Pr(c^*)$ is non-uniform, it is instead interesting to evaluate the expected number of queries required,  Another interesting problem in this case is the design of strategies for finding an approximation of $\CC^*$, investigating the tradeoff between the expected number of mistakes and the number of queries made by the learner. 

\vspace{-0.1in}
\section{Active learning in the realizable case}\label{s:realizable}
\vspace{-0.1in}
%

%We now describe and analyze our active learning algorithm for the noisy realizable setting with persistent noise. 
As a warm up, we start by considering the case where $\lambda = 0$ (no noise).
The underlying cut $c^*$ can be conveniently described by assigning to each node $i$ of $T$ a binary value $y(i) = 0$ if $i$ is above $c^*$, and $y(i) = 1$ if $i$ is below.
% $c^*$. 
Then we can think of an active learning algorithm as querying {\em nodes}, instead of querying pairs of leaves. A query to node $i \in V$ can be implemented by querying {\em any} pair $(x_{i_{\ell}},x_{i_r}) \in L(\lef(i)) \times L(\ri(i))$. 
%where $x_{i_{\ell}}$ is an arbitrary leaf in $L(\lef(i))$, and $x_{i_r}$ is an arbitrary leaf in $L(\ri(i))$. 
When doing so, we actually receive $y(i)$, since for any such $(x_{i_{\ell}},x_{i_r})$, we clearly have $y(i) = \sigma^*(x_{i_{\ell}},x_{i_r})$. 
An obvious baseline is then  to perform a kind of breadth-first search in the tree: We start by querying the root $r$, and observe $y(r)$; if $y(r) = 1$ we stop and output clustering $\hCC = \{L\}$; otherwise, we go down by querying both $\lef(r)$ and $\ri(r)$, and then proceed recursively. It is not hard to show that this simple algorithm will make at most $2K-1$ queries, with an overall running time of $O(K)$, where $K$ is the number of clusters of $\CC(c^*)$. See Figure \ref{f:2} for an illustration. If we know beforehand that $K$ is very small, then this baseline is a tough competitor. Yet, this is not the best we can do in general. Consider, for instance, the line graph in Figure \ref{f:2} (right), where $c^*$ has $K = n$.

\begin{figure}
\vspace{-0.0in}\hspace{-0.1in}
\begin{minipage}{0.62\textwidth}%%%
\vspace{-0.08in}
	%%%\begin{center}
		\begin{subfigure}{.51\textwidth}
		%	\centering
			\includegraphics[width=8.8cm]{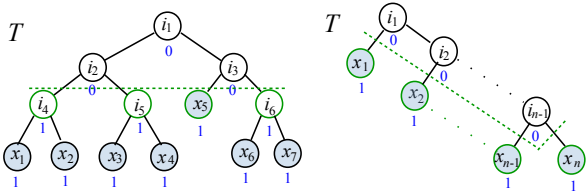}
		\end{subfigure}%	
	%%%\end{center}
	\end{minipage}%%%
%\vspace{-0.3in}
%\quad\quad\quad\quad\quad\quad\quad\quad\quad\quad\quad
%\hspace{-0.02em}
\begin{minipage}[l]{0.39\textwidth}
\vspace{-3.2in}
	\caption{{\bf Left:} The dotted green cut $c^*$ can be described by the set of values of $\{y(i), i \in V\}$, below each node. In this tree, in order to query, say, node $i_2$, it suffices to query any of the four pairs $(x_1,x_3)$, $(x_1,x_4)$, $(x_2,x_3)$, or $(x_2,x_4)$. The 
baseline
%baseline algorithm described in the main text 
queries $i_1$ through $i_6$ in a breadth-first manner, and then stops having identified $c^*$.  \label{f:2}} 
\end{minipage}%
\vspace{-3.38in}
{\small {\bf Right:} This graph has $N(T) = n$. On the depicted cut, the baseline has to query all $n-1$ internal nodes.}
\vspace{-0.15in}
\end{figure}

\vspace{-0.05in}
Ideally, for a given prior 
%information encoded as 
$\Pr(\cdot)$, we would like to obtain a query complexity
%a bound of the number of queries needed to identify $c^*$ 
of the form $\log (1/\Pr(c^*))$, holding in the worst-case for all underlying $c^*$. As we shall see momentarily, this is easily obtained when $\Pr(\cdot)$ is uniform.
% but, even disregarding running times, one cannot hope to achieve this result for all possible prior distributions. 
%In this section, 
We first describe a version space algorithm (One Third Splitting, \ots) %operating on uniform priors, 
that admits a fast implementation, and whose number of queries in the worst-case is $\scO(\log N(T))$. This will in turn pave the way for our second algorithm, Weighted Dichotomic Path (\wdp). \wdp\, leverages $\Pr(\cdot)$, but its theoretical guarantees
% has a query complexity guarantee which 
only hold {\em in expectation} over $\Pr(c^*)$.
%, rather than in the worst-case. 
\wdp\ will then be extended to the persistent noisy setting through its 
%robust 
variant Noisy Weighted Dichotomic Path (\nwdp).

\vspace{-0.05in}
%In order to describe OTS, 
We need a few ancillary definitions. First of all note that, in the noiseless setting, we have a clear hierarchical structure on the labels $y(i)$ of the internal nodes of $T$: Whenever a query reveals a label $y(i) = 0$, we know that all $i$'s ancestors will have label $0$. On the other hand, if we observe $y(i)=1$ we know that all internal nodes of subtree $T(i)$ have label $1$. Hence, disclosing the label of some node indirectly entails disclosing the labels of either its ancestors or its descendants. 
%In what follows we will not distinguish between labels $y(i)$ revealed directly by a query or indirectly by this hierarchical structure.
Given $T$, a bottom-up path is any path connecting a node with one of its ancestors in $T$. In particular, we call a {\em  backbone} path any bottom up path having maximal length. Given $i\in V$, we denote by $S_t(i)$ the {\em version space} at time $t$ associated with $T(i)$, i.e., the set of all cuts of $T(i)$ that are consistent with the labels revealed so far. For any node $j \neq i$, $S_t(i)$ splits into $S_{t}^{y(j)=0}(i)$ and $S_{t}^{y(j)=1}(i)$, the subsets of $S_t(i)$ obtained by imposing a further constraint on $y(j)$.

\vspace{-0.05in}
{\bf \ots~(One Third Splitting):} 
%This algorithm can be described as follows. 
For all $i \in V$, \ots\, maintains over time the value $|S_t(i)|$, i.e., the size of $S_t(i)$, along with the forest $F$ made up of all maximal subtrees $T'$ of $T$ such that $|V(T')|>1$ and for which none of their node labels have been revealed so far. 
%By maximal here we mean that it is not possible to extend any such subtrees by adding a node of $V$ whose label has not already been revealed. 
\ots\ initializes $F$ 
%(when no labels are revealed) 
to contain $T$ only, and maintains $F$ updated over time, by picking any backbone of any subtree $T' \in F$, and visiting it in a bottom-up manner. See the details in Appendix \ref{as:proofs_ots}. The following theorem (proof in Appendix \ref{as:proofs_ots}) crucially relies on the fact that $\pi$ is a backbone path of $T'$, rather than an arbitrary path.

\vspace{-0.07in}
\begin{theorem}\label{t:ots}
On a tree $T$ with $n$ leaves, height $h$, and number of cuts $N$, \ots\, finds $c^*$ by making $\scO(\log N)$ queries. Moreover, an ad hoc data-structure exists that makes the overall running time $\scO(n+h\log N)$ and the space complexity $\scO(n)$.
\end{theorem} 
\vspace{-0.07in}
%
%This algorithm can be implemented using an ad hoc data structure in such a way that the total time required to find $c^*$ is equal to $\scO(n+h_T\log(|S|))$\footnote{We think it is possible to improve this data structure and/or its analysis in order to obtain a better time complexity.}.
%
Hence, Theorem \ref{t:ots} ensures that, for all $c^*$, a time-efficient active learning algorithm exists whose number of queries is of the form $\log(1/\Pr(c^*))$, provided $\Pr(c^*) = 1/N(T)$ for all $c^*$. This query bound is fully in line with well-known results on splittable version spaces~\cite{da05,no11,ml18}, so we cannot make claims of originality. Yet, what is relevant here is that this splitting can be done {\em very efficiently}. 
%The first question we would like to address is whether this extends to all prior distributions. 
%
We complement the above result with a {\em lower} bound holding in expectation over prior distributions on $c^*$. This lower bound depends in a detailed way on the structure of 
%the tree 
$T$.
%at hand and, clearly enough, since it holds in expectation, it also applies to the worst case setting. 
%The following definitions are needed.
%
%As we show next, the answer is negative.
%
%Clearly the above described strategy is able to find $c^*$ with $\scO\left(\log\left(\frac{1}{c^*}\right)\right)$ queries. Interestingly enough, it is easy to show that, in general, it is not possible to achieve this result when the distribution of $c^*$ is non-uniform over all cuts of $T$. 
%
Given tree $T$, with set of leaves $L$, and cut $c^*$, recall the definitions of $\AB(c^*)$ and $\LB(c^*)$ we gave in Section \ref{s:prel}. Let $T'_{c^*}$ be the subtree of $T$ whose nodes are $(\AB(c^*) \cup \LB(c^*)) \setminus L$, and then let\ 
%\vspace{-0.03in}
\(
\kt=\left| L(T'_{c^*}) \right|~
\)
be the number of its leaves. For instance, in Figure \ref{f:2} (left), $T'_{c^*}$ is made up of the six nodes $i_1,\ldots,i_6$, so that $\kt = 3$, while in Figure \ref{f:2} (right), $T'_{c^*}$ has nodes $i_1,\ldots,i_{n-1}$, hence $\kt = 1$. Notice that we always have $\kt\le K$, but for many trees $T$, $\kt$ may be much smaller than $K$. A striking example is again provided by the cut in Figure \ref{f:2} (right), where $\kt = 1$, but $K = n$.
%In particular, when $T$ is a caterpillar tree, $\kt$ is always equal to $1$ while $K$ ranges from $1$ to $n$. 
It is also helpful to introduce $L_{\mathrm{s}}(T)$, the set of all pairs of {\em sibling leaves} in $T$. For instance, in the tree of Figure \ref{f:2}, we have $|L_{\mathrm{s}}(T)| = 3$. One can easily verify that, for all $T$ we have
\(
\max_{c^*} \kt = |L_{\mathrm{s}}(T)| \leq \log_2 N(T)\,.
\)
%
\iffalse
%%%%%%%%%%%%%%%%%%%%%%%%%%%%%%%%
{\bf FV: the following paragraph should be inserted just after the definition of $\kt$ in the main body paper (before the lower bound theorem because it simply follows from the definition of $\kt$).}

Interestingly enough, the definition of $\kt$ implies that, unlike for the uniform prior case, in general there exist input trees $T$ and (non-uniform) prior distributions $\Pr(\cdot)$, such that the total number of queries necessary to find $c^*$ exactly is {\em not} equal to $\scO\left(\log \frac{1}{\Pr(c^*)}\right)$ for {\em all} cuts $c^*$. In fact, from 
the very definition of $\kt$, it immediately follows that for a wide class of trees $T$ we can have (even in expectation over $\Pr(\cdot)$) 
$\kt=\omega(1)$, while there exists a cut $c^*$ such that $\left\lceil\log \frac{1}{\Pr(c^*)}\right\rceil=1$ because $\Pr(c^*)$ can be set arbitrarely close to $1$.
%%%%%%%%%%%%%%%%%%%%%%%%%%%%%%%
\fi
%
%Given any tree $T$, 
%We also denote by $\scP_{>0}$ the set of all prior distributions $\Pr(\cdot)$ over $c^*$ such that, for each cut $c$ of $T$, we have $\Pr(c)>0$. 
We now show that there always exist families of prior distributions $\Pr(\cdot)$ 
such that the expected number of queries needed to find $c^*$ is $\Omega(\E[\kt])$. The quantity $\E[\kt]$ is our notion of {\em average (query) complexity}.
Since the lower bound holds in expectation, it also holds in the worst case. The proof can be found in Appendix \ref{as:lower}.
\vspace{-0.05in}
\begin{theorem}\label{t:LBnoiselessExistsP}
In the noiseless realizable setting, for any tree $T$, any positive integer $B\le |L_{\mathrm{s}}(T)|$, 
%{\em any} prior distribution $\Pr(\cdot)\in\scP_{>0}$ over $c^*$, 
and any (possibly randomized) active learning algorithm $A$, there exists 
%a randomized adversarial strategy $\scS$ for choosing $c^*$ 
a prior distribution $\Pr(\cdot)$ over $c^*$ such that the {\em expected} (over $\Pr(\cdot)$ and $A$'s internal randomization) number of queries $A$ has to make in order to recover $c^*$ is lower bounded by $B/2$, while $B\le\E[\kt]\le 2B$, the latter expectation being over $\Pr(\cdot)$.
\end{theorem}
\vspace{-0.05in}

Next, we describe an algorithm that, unlike \ots, is indeed able to take advantage of the prior distribution, but it does so at the price of bounding the number of queries only {\em in expectation}.
% over $c^*$.

\vspace{-0.03in}
{\bf \wdp~(Weighted Dichotomic Path)}: Recall prior distibution (\ref{e:prob}), collectively encoded through the values $\{\Pr(i)$, $i \in V\}$. As for \ots, we denote by $F$ the forest made up of all maximal subtrees $T'$ of $T$ such that $|V(T')|>1$ and for which none of their node labels have so far been revealed. $F$ is updated over time, and initially contains only $T$. We denote by $\pi(u, v)$ a bottom-up path in $T$ having as terminal nodes $u$ and $v$ (hence $v$ is an ancestor of $u$ in $T$). For a given cut $c^*$, and associated labels $\{y(i), i \in V\}$, any tree $T' \in F$, and any node $i\in V(T')$, we define\footnote
{
For 
%the sake of 
definiteness, we set $y(\pa(r)) = 0$, that is, we are treating the parent of $r(T)$ as a \enquote{dummy super--root} with labeled $0$ since time $t=0$. Thus, according to this definition, $q(r)=\Pr(r)$.
}
\begin{equation}\label{e:qi}
q(i) = \Pr(y(i)=1\wedge y(\pa(i))=0)%=\Pr((i,\pa(i))\in c^*)
     = \Pr(i)\cdot\prod_{j\in\pi(\pa(i), r(T'))}(1-\Pr(j))~.
\end{equation}
%where $r'$ is the root of $T'$ and $\pa(i)$ is the parent of node $i$
%
We then associate with any backbone path of the form $\pi(\ell, r(T'))$, where $\ell \in L(T')$, an entropy 
$H(\pi(\ell, r(T'))) = -\sum_{i\in \pi(\ell, r(T'))} q(i)\,\log_2 q(i)$. 
Notice that at the beginning we have $\sum_{i\in \pi(\ell, r(T))} q(i) = 1$ for all $\ell \in L$. This invariant will be maintained on all subtrees $T'$. The prior probabilities $\Pr(i)$ will evolve during the algorithm's functioning into posterior probabilities based on the information revealed by the labels. Accordingly, also the related values $q(i)$ w.r.t. which the entropy $H(\cdot)$ is calculated will change over time. 
%Throughout the rest of section (as well as in the associated Appendix \ref{as:proofs_realizable}), any occurrence of $\Pr(i)$ has to be interpreted as the posterior probability of $i$, unless otherwise specified. Likewise, any occurrence of $q(i)$ and entropy $H(\cdot)$ will be meant to be defined through this posterior $\Pr(\cdot)$. 

\vspace{-0.05in}
Due to space limitations, \wdp's pseudocode is given 
%as Algorithm \ref{a:wdp} 
in Appendix \ref{as:wdp}, but we have included an example of its execution in Figure \ref{fig:wdp}. At each round, \wdp\ finds the path 
%$\pi(\ell, r')$ 
whose entropy is maximized over all bottom-up paths $\pi(\ell, r')$, with $\ell \in L$ and $r' = r(T')$, where $T'$ is the subtree in $F$ containing $\ell$.
%backbone paths.
% (with ties broken arbitrarily).
%over all the bottom-up paths of all trees in $F$ . 
%Let $T' \in F$ be the subtree containing the selected path, so that $r' = r(T')$. 
%Note that, by the definition of $F$ and \enquote{entropy of a path}, $\ell$ must be a leaf of $T'$ and $r'$ must be its root, because \wdp\ selectes the bottom-up path having maximum entropy. 
\wdp\ performs a binary search on such $\pi(\ell,r')$ to find the edge of $T'$ which is cut by $c^*$, %and lies on this path. 
taking into account the current values of $q(i)$ over that path.
Once a binary search terminates, \wdp\ updates $F$ and the probabilities $\Pr(i)$ at all nodes $i$ in the subtrees of $F$. 
%so as to reflect the new knowledge gathered by the queried labels. 
%
See Figure \ref{fig:wdp} for an example. Notice that the
%probabilities 
$\Pr(i)$ on the selected path become either $0$ (if above the edge cut by $c^*$) or $1$ (if below). In turn, this causes updates on all probabilities $q(i)$. 
\wdp\, continues with the
%by performing the 
next binary search on the next path with maximum entropy at the current stage, discovering another edge cut by $c^*$, and so on, until $F$ becomes empty.
Denote by $\scP_{>0}$ the set of all priors $\Pr(\cdot)$ such that for all cuts $c$ of $T$ we have $\Pr(c)>0$. The proof of the following theorem is given in Appendix \ref{as:wdp}.
\begin{figure}[!t]
%\begin{center}
\begin{minipage}{0.65\textwidth}%%%
\vspace{-0.0in}\hspace{-0.18in}
		\begin{subfigure}{.6\textwidth}
			%\centering
			\includegraphics[width=4.7cm]{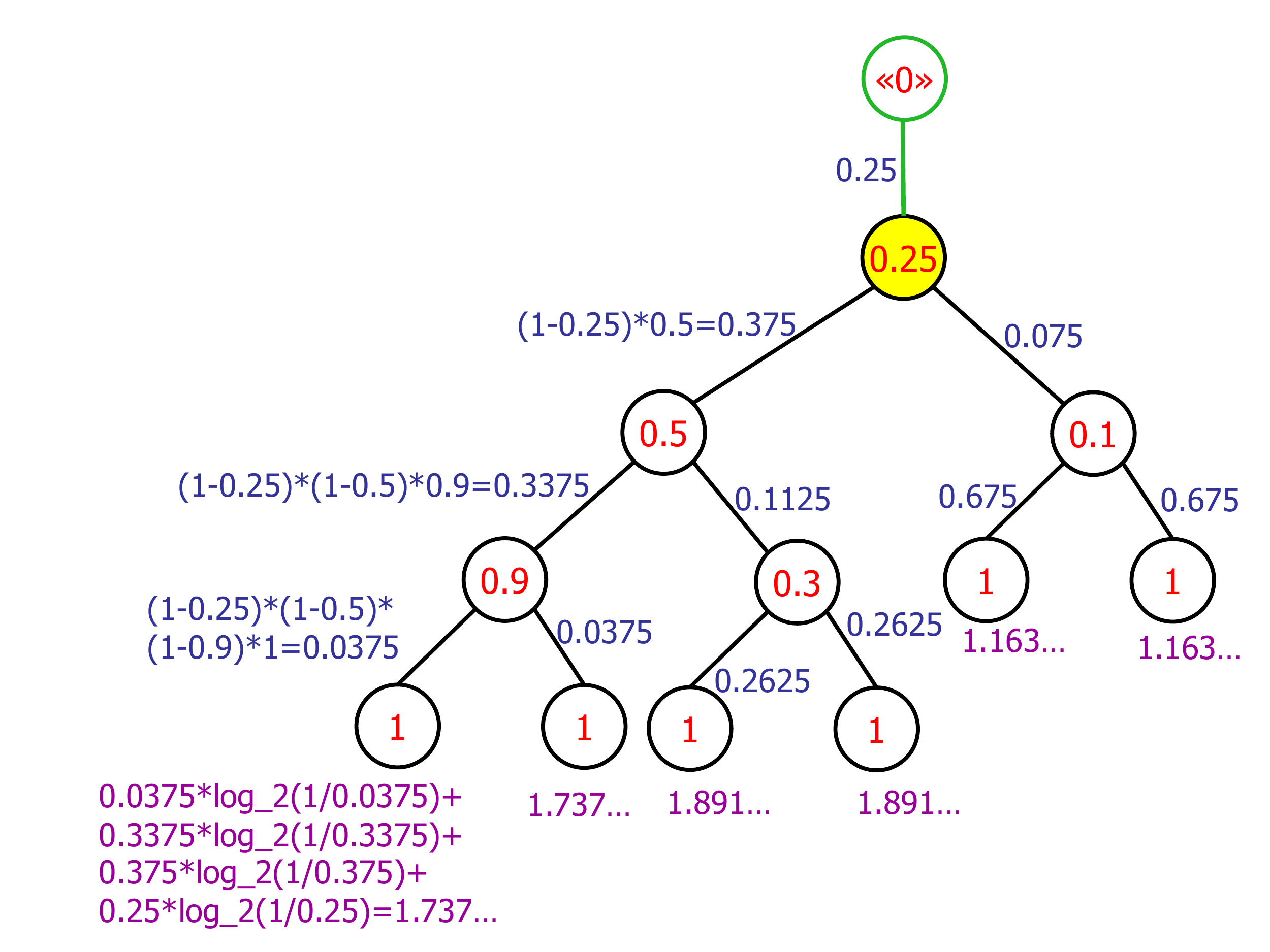}
			%\caption{No labels are revealed.}
			\label{fig:s-20-23}
		\end{subfigure}%
		\begin{subfigure}{.9\textwidth}
			%\centering
			\hspace{-0.4in}
			\includegraphics[width=4.7cm]{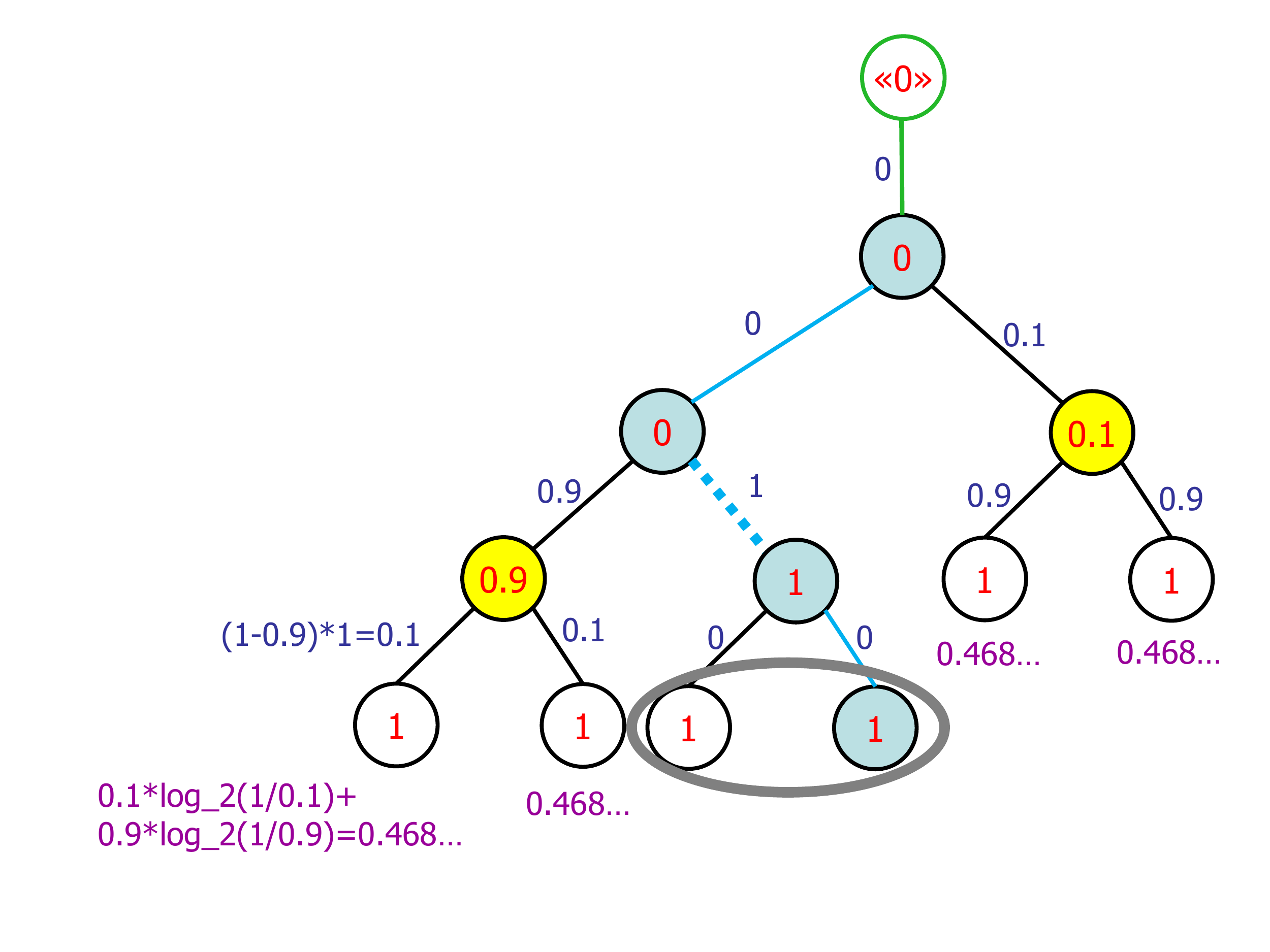}
			%\caption{After a binary search.}
			\label{fig:s-500-480}
		\end{subfigure}
%		\begin{subfigure}{.5\textwidth}
%			\centering
%			\includegraphics[width=1\linewidth]{plots/real-world-par=2.pdf}
%			\caption{Dataset RW-2265-3939.}
%			\label{fig:rw-2265-3939}
%		\end{subfigure}
\end{minipage}%%%
%\end{center}
\begin{minipage}[r]{0.35\textwidth}
\vspace{0.05in}
\caption{An example of input tree $T$ before (left) and after (right) the first binary search of \wdp. The green node is a dummy super-root.
% which turns out to be useful for the description of the algorithm. 
The nodes in yellow are the roots of the subtrees currently included in forest $F$. The numbers in red within each node $i$ indicate the probabilities $\Pr(i)$, while the $q(i)$ values are in blue, and viewed here as 
%more naturally 
associated with edges $(\pa(i),i)$. The\label{fig:wdp}}
\end{minipage}%
\vspace{-0.1in}
{\small 
%for all nodes $i\in V$. and they encode the probability of the true cut $c^*$ over all cuts of $T$.  
%For each $i\in V$, the blue numbers indicate the probabilities  $\Pr(y(i)=1\wedge y(\pa(i))=0)=\Pr((i,\pa(i))\in c^*)$, i.e. the probabilities that each edge $(i, \pa(i))$ for all $i\in V$ is cut by $c^*$.
%No labels are revealed. After a binary search.
magenta numbers at each leaf $\ell$ give the entropy $H(\pi(\ell,r(T')))$, where $r(T')$ is the root of the subtree in $F$ that\ \ contains both $\ell$ and $r(T')$. 
{\bf Left:} The input tree $T$ at time $t=0$. No labels are revealed, 
%and no binary searches are performed. 
and no clusters of $\CC(c^*)$ are found. 
{\bf Right:} Tree $T$ after a full binary search has been performed on the depicted light blue path. Before this binary search, that path connected a leaf of a subtree in $F$ to its root (in this case, $F$ contains only $T$). The selected path is the one maximazing entropy within the forest/tree on the left. The dashed line indicates the edge of $c^*$ found by the binary search. The red, blue and magenta numbers are updated accordingly to the result of the binary search. The leaves enclosed in the grey ellipse are now known to form a cluster of $\CC(c^*)$.}
\vspace{-0.2in}
\end{figure}
%
%We now give our query complexity guarantees for \wdp\ in the noiseless setting.
%
\vspace{-0.05in}
\begin{theorem}\label{t:wdp1}
In the noiseless realizable setting, for any tree $T$ of height $h$, any prior distribution $\Pr(\cdot)$ over $c^*$, such that $\Pr(\cdot) \in \scP_{>0}$, the expected number of queries made by \wdp\ to find $c^*$ is
\(
\scO\left(\E\left[\kt\right]\log h\right),
\) 
the expectations being over $\Pr(\cdot)$.
\end{theorem}
\vspace{-0.05in}
For instance, in the line graph of Figure \ref{f:2} (right), the expected number of queries is $\scO(\log n)$ for any prior $\Pr(\cdot)$, while if $T$ is a complete binary tree with $n$ leaves, and we know that $\CC(c^*)$ has $\scO(K)$ clusters, we can set $\Pr(i)$ in (\ref{e:prob}) as $\Pr(i) = 1/\log K$, which would guarantee $\E[\kt] = \scO(K)$, and a bound on the expected number of queries of the form $\scO(K\log\log n)$. By comparison, observe that the results in \cite{da05,gk17,td17} would give a query complexity which is at best $\scO(K\log^2 n)$, while those in \cite{no11,ml18} yield at best $\scO(K\log n)$. In addition, we show below (Remark \ref{r:time}) that our algorithm has very compelling running time guarantees.

\vspace{-0.05in}
It is often the case that a linkage function generating $T$ also tags each internal node $i$ with a {\em coherence level} $\alpha_i$ of $T(i)$, which is typically increasing as we move downwards from root to leaves. A common situation in hierarchical clustering is then to figure out the ``right'' level of granularity of the flat clustering we search for by defining parallel {\em bands} of nodes of similar coherence where $c^*$ is possibly located. For such cases, a slightly more involved guarantee for \wdp\, is contained in Theorem \ref{t:wdp2} in Appendix \ref{as:wdp}, where the query complexity depends in a more detailed way on the interplay between $T$ and the prior $\Pr(\cdot)$. In the above example, if we have $b$-many edge-disjoint bands, Theorem \ref{t:wdp2} replaces factor $\log h$ of Theorem \ref{t:wdp1} by $\log b$.

\vspace{-0.03in}
{\bf \nwdp~(Noisy Weighted Dichotomic Path)}: This is a robust variant of \wdp\ that copes with persistent noise. 
%In the setting of random {\em persistent} noise, we assume that a subset $\Lambda \subset \bnm{L}{2}$ is formed by selecting uniformly at random %$\widetilde{Q}$-many a certain number of pairs of $T$'s leaves from $\bnm{L}{2}$, where we denote by $\bnm{L}{2}$ the set of all pairs of leaves of $L$. 
 %We have therefore $|\Lambda|=\widetilde{Q}$.
%Since we assume the noise is persistent, for each query involving a pair of leaves belonging to $\Lambda$, the answer obtained by the learner is always incorrect. We denote by $\lambda$ the fraction of pairs of leaves of $L$ whose queries are noisy, i.e. $\lambda=\frac{|\Lambda|}{\bnm{n}{2}}$.
%\footnote{We recall that a query can be represented as a pair of leaves $(\ell,\ell')\in L$ and the answer is either \enquote{$\exists C\in\CC^* : \ell,\ell'\in C$}, or 
% \enquote{$\exists C,C'\in\CC^*:\ell\in C,\ell'\in C',C\neq C'$}, i.e. either \enquote{$\ell$ and $\ell'$ belong  to the same cluster in $\CC^*$} or \enquote{$\ell$ and $\ell'$ do not belong to the same cluster in $\CC^*$}.}
%
Whenever a label $y(i)$ is requested, \nwdp\ determines its value by a majority vote over randomly selected pairs from $L(\lef(i))\times L(\ri(i))$. Due to space limitations, all details are contained in Appendix \ref{as:nwdp}.
%
%More sophisticated techniques involve asking for the labels of (close) descendants and ancestors of $i$. In particular one can take into account that {\bf (i)} if $y(i)=0$ we must have $y(a)=0$ for all ancestors $a$ of $i$, and {\bf (ii)} if $y(i)=1$ we must have $y(d)=0$ for all descendants $d$ of $i$. The rationale for this \enquote{answer reinforcement} is the attempt to minimize the number of mistakes made w.h.p.\footnote{We use in our statements the acronym w.h.p. to signify with probability at least $1-\scO(n^{-1})$ as $n\to\infty$.} by \wdp, viz. to minimize the {\em distance} between $\CC^*$ and the clustering output by \wdp, which can be measured in different ways (see Section~\ref{s:errors}). We call \nwdp~(Noisy Weighted Dichotomic Path) this version of \wdp. 
%
The next theorem quantifies \nwdp's performance in terms of a tradeoff between the expected number of queries and the distance to the noiseless ground-truth matrix $\Sigma^*$.
%
% $d_H(\Sigma,\scC_{\nwdp})$ between the ground truth $\Sigma$ and the clustering $\scC_\nwdp$ it outputs.
%
\vspace{-0.07in}
\begin{theorem}\label{t:nwdp}
In the noisy realizable setting, given any input tree $T$ of height $h$, any cut $c^* \sim \Pr(\cdot)\in\scP_{>0}$, and any $\delta\in(0,1/2)$, \nwdp\, outputs with probability $\geq 1-\delta$ (over the noise in the labels) a clustering $\hCC$ such that $\frac{1}{n^2}\,d_H(\Sigma^*,\hCC)=\scO\left(\frac{1}{n}\,\frac{(\log(n/\delta))^{3/2}}{(1-2\lambda)^{3}}\right)$ by asking
$\scO\left(\frac{\log(n/\delta)}{(1-2\lambda)^{2}}\,\E\kt\log h\right)$ queries in expectation (over $\Pr(\cdot)$).
\end{theorem}
\vspace{-0.07in}
%
%that are sufficient to have $d_H(\Sigma,\scC_{\nwdp})=\scO\left(\frac{n(\log(n/\delta))^{3/2}}{(1-2\lambda)^{3}}\right)$, is equal with probability at least $1-\delta$ over the noise randomness to $\scO\left(\frac{(\log(n/\delta)}{(1-2\lambda)^{2}}\kt\log\left(h\right)\right)$.

\vspace{-0.07in}
\begin{remark}\label{r:time}
Compared to the query bound in Theorem \ref{t:wdp1}, the one in Theorem \ref{t:nwdp} adds a factor due to noise. The very same extra factor is contained in the bound of \cite{ms17b}.
Regarding the running time of \wdp\,, the version we have described can be naively implemented to run in $\scO(n\,\E\kt)$ expected time overall.
%requires $\scO(n)$ space and (in expectation over $\Pr(\cdot)$)  time. We stress that there exists 
A more time-efficient variant of \wdp\ exists for which Theorem~\ref{t:wdp1} and Theorem~\ref{t:wdp2} still hold, that requires 
%$\scO(n)$ space and 
$\scO(n+h\,\E\kt)$ expected time. Likewise, an efficient variant of \nwdp\ exists for which Theorem~\ref{t:nwdp} holds, that takes 
$\scO\left(n+\left(h+\frac{\log^2 n }{(1-2\lambda)^{2}} \right)\E\kt\right)$ expected time.
%$\scO\left(n+\left(h+\frac{\log(n/\delta)}{(1-2\lambda)^{2}}\,\log h\right)\E\kt\right)$ expected time overall.
%Regarding the computational complexity, there exists an efficient implementation of a variant of (in expectation over $\Pr(\cdot)$) and $\scO(n)$ space.
\end{remark}
\vspace{-0.05in}

\newcommand{\nr}{{\sc{nr}}}

\vspace{-0.1in}
\section{Selective sampling in the non-realizable case}\label{s:nonrealizable}
\vspace{-0.1in}
In the non-realizable case, we adapt to our clustering scenario the importance-weighted algorithm in \cite{b+10}. The algorithm is a selective sampler that proceeds in a sequence of rounds $t = 1, 2,\ldots $. In round $t$ a pair $(x_{i_t},x_{j_t})$ is drawn at random from distribution $\DD$ over the entries of a given ground truth matrix $\Sigma$, and the algorithm produces in response a probability value $p_t = p_t(x_{i_t},x_{j_t})$. A Bernoulli variable $Q_t \in \{0,1\}$ is then generated with $\Pr(Q_t = 1) = p_t$, and if $Q_t =1$ the label $\sigma_t = \sigma(x_{i_t},x_{j_t})$ is queried, and the algorithm updates its internal state; otherwise, we skip to the next round. The way $p_t$ is generated is described as follows. Given tree $T$, the algorithm maintains at each round $t$ an importance-weighted empirical risk minimizer cut ${\hat c_t}$, defined as \ 
\(
{\hat c_t} = \argmin_{c} \err_{t-1}(\CC(c))~,
\)
where the ``argmin'' is over all cuts $c$ realized by $T$, and
\(
\err_{t-1}(\CC) = \frac{1}{t-1} \sum_{s=1}^{t-1} \frac{Q_s}{p_s}\,\{ \sigma_{\CC}(x_{i_s},x_{j_s})  \neq \sigma_s  \}~,
\)
being $\{\cdot\}$ the indicator function of the predicate at argument. This is paired up with a {\em perturbed} empirical risk minimizer
\(
{\hat c'_t} = \argmin_{c\,:\, \sigma_{\CC(c)}(x_{i_t},x_{j_t}) \neq \sigma_{\CC(\hat c_t)}(x_{i_t},x_{j_t}) } \err_{t-1}(\CC(c))~,
\)
the ``argmin'' being over all cuts $c$ realized by $T$ that disagree with ${\hat c_t}$ on the current pair $(x_{i_t},x_{j_t})$. The value of $p_t$ is a function  of
% the empirical risk difference 
$d_t = \err_{t-1}(\CC({\hat c'_t})) - \err_{t-1}(\CC({\hat c_t}))$, of the form 
%The functional form of $p_t$ is
%\footnote
%{
%We refer the reader to Section 4 in \cite{b+10} for details.
%}
\begin{equation}\label{e:pt}
p_t = \min \left\{1,\scO\left(1/d_t^2 + 1/d_t\right) \log ((N(T)/\delta)\log t)/t \right\}~,
\end{equation}
where $N(T)$ is the total number of cuts realized by $T$ (i.e., the size of our comparison class), and $\delta$ is the desired confidence parameter.
% in the high probability statements. 
Once stopped, say in round $t_0$, the algorithm gives in output cut ${\hat c_{t_0+1}}$, and the associated clustering $\CC({\hat c_{t_0+1}})$. Let us call the resulting algorithm \nr\ (Non-Realizable).

Despite $N(T)$ can be exponential in $n$,
%the number $n$ of leaves, 
there are very efficient ways of computing ${\hat c_t}$, ${\hat c'_t}$, and hence $p_t$ at each round. In particular, an ad hoc procedure exists that incrementally computes these quantities
%${\hat c_t}$, ${\hat c'_t}$, and $p_t$ 
by leveraging the sequential nature of \nr.
For a given $T$, and constant $K \geq 1$, consider the class $\CCC(T,K)$ of cuts inducing clusterings with at most $K$ clusters. Set $R^* = R^*(T,\DD) = \min_{c \in \CCC(T,K)} \Pr_{(x_i,x_j) \sim \DD}\left( \sigma(x_i,x_j) \neq \sigma_{\CC(c)}(x_i,x_j) \right)$, and $B_{\delta}(K,n) = K\log n+\log(1/\delta)$. The following theorem is an adaptation 
%to our setting 
of a result in \cite{b+10}. See Appendix \ref{as:proofs_nonrealizable} for a proof.
\vspace{-0.05in}
\begin{theorem}\label{t:nonrealizable}
Let $T$ have $n$ leaves and height $h$. Given confidence parameter $\delta$, 
%in the selective sampling setting of this section, 
for any $t \geq 1$, with probability at least $1-\delta$, the excess risk (\ref{e:excessrisk}) achieved by the clustering $\CC({\hat c_{t+1}})$ computed by \nr\, w.r.t. the best cut in class $\CCC(T,K)$ is bounded by
\(
\scO\left(\sqrt{\frac{B_{\delta}(K,n)\log t}{t}} + \frac{B_{\delta}(K,n)\log t}{t}\right),
\)
while the (expected) number of labels $\sum_{s = 1}^t p_s$ is bounded by
\(
\scO\left( \theta \left(R^* t + \sqrt{t\,B_{\delta}(K,n)\log t} + B_{\delta}(K,n)\log^3 t \right) \right),
\)
where $\theta = \theta(\CCC(T,K),\DD)$ is the disagreement coefficient of $\CCC(T,K)$ w.r.t. distribution $\DD$. In particular, when $\DD$ is uniform we have 
%over the entries of $\Sigma$, 
$\theta \leq K$. Moreover, there exists a fast implementation of \nr\ whose expected running time per round is $\E_{(x_i,x_j) \sim \DD} [\de(\lca(x_i,x_j))] \leq h$, where $\de(\lca(x_i,x_j))$ is the depth in $T$ of the lowest common ancestor of $x_i$ and $x_j$. 
\end{theorem}
\vspace{-0.05in}

\newcommand{\sing}{{\sc{sing}}}
\newcommand{\comp}{{\sc{comp}}}
\newcommand{\med}{{\sc{med}}}
\newcommand{\nwdpu}{{\sc{n-wdp-unif}}}
\newcommand{\nwdpe}{{\sc{n-wdp-exp}}}
\newcommand{\erm}{{\sc{erm}}}
\newcommand{\breadthfirst}{{\sc{bf}}}
\newcommand{\best}{{\sc{best}}}

\vspace{-0.12in}
\section{Preliminary experiments}\label{s:exp}
\vspace{-0.11in}
The goal of these experiments was to contrast active learning methods originating from the persistent noisy setting (specifically, \nwdp) to those originating from the non-realizable setting (specifically, \nr). The comparison is carried out on the hierarchies produced by standard HC methods operating on the first $n=10000$ datapoints in the well-known MNIST dataset from \url{http://yann.lecun.com/exdb/mnist/}, yielding a sample space of $10^8$ pairs. We used Euclidean distance combined with the single linkage (\sing), median linkage (\med), and complete linkage (\comp) functions. The $n\times n$ ground-truth matrix $\Sigma$ is provided by the 10 class labels of MNIST. 

We compared \nwdp\ with uniform prior and \nr\ to two baselines: passive learning based on empirical risk minimization (\erm), and the active learning baseline performing breadth-first search from the root (\breadthfirst, Section \ref{s:realizable}) made robust to noise as in \nwdp.
%by repeating queries at different pairs of leaves. 
For reference, we also computed for each of the three hierarchies the performance of the best cut in hindsight (\best) on the {\em entire} matrix $\Sigma$.
% w.r.t. $d_H$. 
That is essentially the best one can hope for in each of the three cases. 
All algorithms except \erm\ are randomized and have a single parameter to tune. We let such parameters vary across suitable ranges and, for each algorithm, picked the best performing value on a validation set of 500 labeled pairs.

In Table \ref{t:stats}, we have collected relevant statistics about the three hierarchies. In particular, the single linkage tree turned out to be very deep, while the complete linkage one is quite balanced.
We evaluated test set accuracy vs. number of queries after parameter tuning, excluding these 500 pairs. For \nwdp, once a target number of queries was reached, we computed as current output the maximum-a-posteriori cut. In order to reduce variance, we repeated each experiment 10 times.

\begin{table}
\begin{small}
\begin{center}
  \begin{tabular}{l|r|r|r|r}
    Tree      & Avg depth        &Std. dev       &\best's error      	&\best's $K$\\
\hline
    \sing     &2950              &1413.6    	 &8.26\%		&4679\\
    \med      &186.4             &41.8           &8,51\%		&1603\\
    \comp     &17.1              &3.3            &8.81\%		&557
  \end{tabular}
\end{center}
\end{small}
\vspace{-0.05in}
\caption{Statistics of the trees used in our experiments. These trees result from applying the linkage functions \sing, \comp, and \med\ to the MNIST dataset (first 10000 samples). Each tree has the same set of $n = 10000$ leaves. ``Avg depth'' is the average depth of the leaves in the tree, ``Std. dev'' is its standard deviation. For reference, we report the performance of \best\ (i.e., the minimizer of $d_H$ over all possible cuts realized by the trees), along with the associated number of clusters $K$.\label{t:stats}
}
\vspace{-0.3in}
\end{table}

%
\iffalse
%%%%%%%%%%%%%%%%%%%%%%%%%%%%%%%%%%%%%%%%%%%%%%%%%%%
\begin{figure*}[t]
\vspace{-0.0in}
\begin{minipage}[l]{0.25\textwidth}
%\begin{table}
\begin{small}
  \begin{tabular}{l|r|r}
    Tree & Avg $h$  &$K$\\
\hline
    \sing     &2950      &4679\\
    \comp     &17.1      &557\\
    \med      &186.4     &1603
  \end{tabular}
\end{small}
%\end{table}
%
\end{minipage}%
\begin{minipage}{0.75\textwidth}%%%
\vspace{-0.08in}
	\begin{center}
\begin{tabular}{c c c}
               \includegraphics[scale=0.15, trim= 5 10 10 5, clip=true]{result1.pdf} &
\hspace*{-5mm} \includegraphics[scale=0.15, trim= 5 10 10 5, clip=true]{result1.pdf} & 
\hspace*{-5mm} \includegraphics[scale=0.15, trim= 5 10 10 5, clip=true]{result1.pdf}
\end{tabular}
\end{center}
\end{minipage}%%%
%\vspace{-0.3in}
%\quad\quad\quad\quad\quad\quad\quad\quad\quad\quad\quad
%\hspace{-0.02em}
\vspace{-0.0in}
\caption{{\bf Left:} Statistics of the trees resulting from applying the linkage functions \sing, \comp, and \med. Avg $h$ is the average depth (height) of leaves; $K$ is the number of clusters computed by \best. {\bf Right:} Test error vs. no. of queries (log scale) for the various algorithms on the three hierarchies. For reference, we also plot the test error performance of the best cut in hindsight.\label{f:results}}
\vspace{-0.1in}
\end{figure*}
%%%%%%%%%%%%%%%%%%%%%%%%%%%%%%%%%%%%%%%%%%%%%%%%%%%%
\fi

%
The details of our empirical comparison are contained in Appendix \ref{ass:exp}. Though our experiments are quite preliminary, 
%(details contained in Appendix \ref{ass:exp}), 
some trends can be readily spotted: 
i. \nwdp\ significantly outperforms \nr. E.g., in \comp\ at 250 queries, the test set accuracy of \nwdp\ is at 9.52\%, while \nr\ is at 10.1\%. A similar performance gap at low number of queries one can observe in \sing\ and \med.
% in terms of trade-off of statistical accuracy vs. number of queries. 
This trend was expected: \nr\ is very conservative, as it has been designed to work under more general conditions 
%(non-realizable scenario, arbitrary distribution $\DD$) 
%more general than uniform 
than \nwdp. We conjecture that, whenever the specific task at hand allows one to make an aggressive noise-free 
%active learning 
algorithm (like \wdp) robust to persistent noise (like \nwdp), this outcome is quite likely to occur. %\nwdp\ deals with noise by repeatedly querying pairs of samples with the same lowest-common ancestor. \nr\ is not designed this way.
%
%ii. Different trees give different results i.t.o. approximation error. 
%In all cases, \nwdp\ tend to reach a very similar statistical accuracy as \best\ after relatively few queries.
%For instance, in \sing, \nwdpu\ is at 8.3\% test error just after 1000 queries, while \best\ is at 8.26\%.  
%
ii. \breadthfirst\, is competitive only when \best\ has few clusters.
iii. \nwdp\ clearly outperforms \erm, while the comparison between \nr\ and \erm\ yields mixed results.

%\vspace{-0.05in}
\noindent{\bf Ongoing activity. }
Beyond presenting new algorithms and analyses for pairwise similarity-based active learning, our goal was to put different approaches to active learning on the same footing for comparison on real data. Some initial trends are suggested by our experiments, but a more thorough investigation is underway. We are currently using other datasets, of different nature and size. Further HC methods are also under consideration, like those based on $k$-means.

\newpage

\appendix

\section{Missing material from Section \ref{s:intro}}\label{as:comp}

\subsection{Further related work}
Further papers related to our work are those dealing with clustering with queries, e.g., \cite{d+14,akbd16,ms17b,bb08,abv17}. In \cite{d+14} the authors show that $\scO(Kn)$ similarity queries are both necessary and sufficient to achieve exact reconstruction of an {\em arbitrary} clustering with $K$ clusters on $n$ items. This is generalized by \cite{ms17b} where persistent random noise is added. \cite{akbd16} assume the feedback is center-based with a margin condition on top. Because we are constrained to a clustering produced by cutting a given tree, the results in \cite{d+14,akbd16,ms17b} are incomparable to ours, due to the different assumptions. 
In \cite{bb08,abv17}) the authors consider clusterings realized by a given comparison class (as we do here). Yet, the queries they are allowing are different from ours, hence their results are again incomparable to ours.

\section{Missing material from Section \ref{s:realizable}}\label{as:proofs_realizable}

\subsection{One Third Splitting (\ots)}\label{as:proofs_ots}

%This algorithm can be described as follows. 
For all $i \in V$, \ots\, maintain over time the value $|S_t(i)|$, i.e., the size of $S_t(i)$, along with the forest $F$ made up of all maximal subtrees $T'$ of $T$ such that $|V(T')|>1$ and for which none of their node labels have been revealed so far. Notice that we will not distinguish between labels $y(i)$ revealed directly by a query or indirectly by the hierarchical structure. By maximal here we mean that it is not possible to extend any such subtrees by adding a node of $V$ whose label has not already been revealed. \ots\ initializes $F$ (when no labels are revealed) to contain $T$ only, and maintains $F$ updated over time. Let subtree $T' \in F$ be arbitrarily chosen, and $\pi$ be any backbone path of $T'$. At time $t$, OTS visits $\pi$ in a bottom-up manner, and finds the {\em lowest} node $i^*_t$ in this path satisfying $\left|S_{t}^{y(i^*_t)=1}(r)\right|\le 2\left|S_{t}^{y(i^*_t)=0}(r)\right|$, i.e., $\left|S_{t}^{y(i^*_t)=1}(r)\right|\le \frac{2}{3}|S_t(r)|$, then query node $i^*_t$. We repeat the above procedure until $|S_t(r)| = 1$, i.e., until we find $c^*$. 

The next lemma is key to showing the logarithmic number of queries made by \ots.

\begin{lemma}\label{l:ots}  
With the notation introduced in Section \ref{s:realizable}, at each time $t$, the query $y(i_t^*)$ made by \ots\, splits the version space $S_t(r)$ in such a way that\footnote
{
This bound is indeed tight for this strategy when the input is a full binary tree of height $3$.
} 
\[
\min\left\{\left|S_{t}^{y(i_t^*)=0}(r)\right|, \left|S_{t}^{y(i_t^*)=1}(r)\right|\right\}\ge \frac{|S_t(r)|}{3}~.
\]
\end{lemma}
%
%FV I have some latex problems when using \begin{proof} and \end{proof}
\begin{proof}
At each time $t$, 
%let $r_1, r_2, \ldots, r_{|F|}$ be the roots of the trees 
%$T_1, T_2, \ldots, T_{|F|}$
%in forest $F$, in an arbitrarily order. We have therefore, for all time $t$, 
%$|S_t(r)|=\prod_{k=1}^{|F|}|S_t(r_k)|$. 
$|S_t(r)|$ is the product of the cardinality of 
$S_t(\widetilde{r})$ over all roots $\widetilde{r}$ of the trees currently contained in $F$. 
Let $\pi$ be a backbone of one such tree, say tree $T'$, with root $r'$. Since $T'$ is arbitrary, in order to prove the statement, it is sufficient to show that 
\[
\min\left\{\left|S_{t}^{y(i_t^*)=0}(r')\right|, \left|S_{t}^{y(i_t^*)=1}(r')\right|\right\}\ge \frac{|S_t(r')|}{3}~.
\]
Let $h(\pi)$ be the length of $\pi$, i.e., the number of its edges, and $\langle j_0, j_1, \ldots, j_{h(\pi)}\rangle$ be
the sequence of its nodes, from bottom to top. %, i.e. $j_0\in L(r')$ and $j_{h(\pi)}=r'$.
For any $k<h(\pi)$, we denote by $j^{\mathrm{s}}_{k}$ the sibling of $j_k$ in $T'$ (hence, by this definition $j^{\mathrm{s}}_{k}$ does {\em not} belong to $\pi$).
Now, observe that the number of possible labelings of $\pi$ is equal to $h(\pi)+1$, that is, each labeling of $\pi$ corresponds to an integer $z\in\{0, 1, \ldots, h(\pi)\}$ such that $y(j_{k})=1$ for all $k\le z$ and $y(j_{k})=0$ for all $z<k\le h(\pi)$. Then,
given any labeling of the nodes of $\pi$ (represented by the above $z$), we have 
$$
|S_t(r')| 
= 
\begin{cases}
\prod_{k=z}^{h(\pi)-1} |S_t(j^{\mathrm{s}}_{k})|   &{\mbox{if }} z<h(\pi)~,\\ 
1                                                  &{\mbox{if }}  z=h(\pi)~.
\end{cases}
$$
In fact, the disclosure of all labels of the nodes in $\pi$ when $z<h(\pi)$ would decompose $T'$ into $(h(\pi)-z)$-many subtrees whose labelings are independent of one another. 
For all $z\in\{0, 1, \ldots, h(\pi)-1\}$, let us denote for brevity $\prod_{k=z}^{h(\pi)-1} |S_t(j^{\mathrm{s}}_{k})|$ by $\scS_z$, and also denote for convenience $|S_t^{y(r')=1}(r')|$ by $\scS_{h(\pi)}$ Notice that, by definition, $|S_t^{y(r')=1}(r')| = 1$, and corresponds to the special case $z=h(\pi)$. With this notation, it is now important to note that $i_t^*$ must be 
the parent of $j^{\mathrm{s}}_{z^*}$, for some $z^*\in\{0, 1, \ldots, h(\pi)-1\}$, and that $\left|S_{t}^{y(i_t^*)=0}(r')\right|=\scS_{z^*}$.
Thus, taking into account all possible $(h(\pi)+1)$ labelings of $\pi$, the cardinality of $S_t(r')$ can be written as follows:
\[
|S_t(r')|=\sum_{z=0}^{h(\pi)}\scS_z~.
\]
%which takes into account all possible $(h(\pi)+1)$-many labelings of $\pi$.

At this point, by definition, we have:
\begin{itemize}
\item[(i)]  $\scS_0=\scS_1$, as $|S_t(j^{\mathrm{s}}_0)|=1$, which in turn implies $\max_z \scS_z\le \frac{|S_t(r')|}{2}$, and 
\item[(ii)] $\scS_z\ge \scS_{z+1}$ for all $z\in\{0,\ldots,h(\pi)-1\}$~.
\end{itemize}
See Figure~\ref{f:OTSfigure} for a pictorial illustration.

\begin{figure}
\center
\includegraphics[width=10cm]{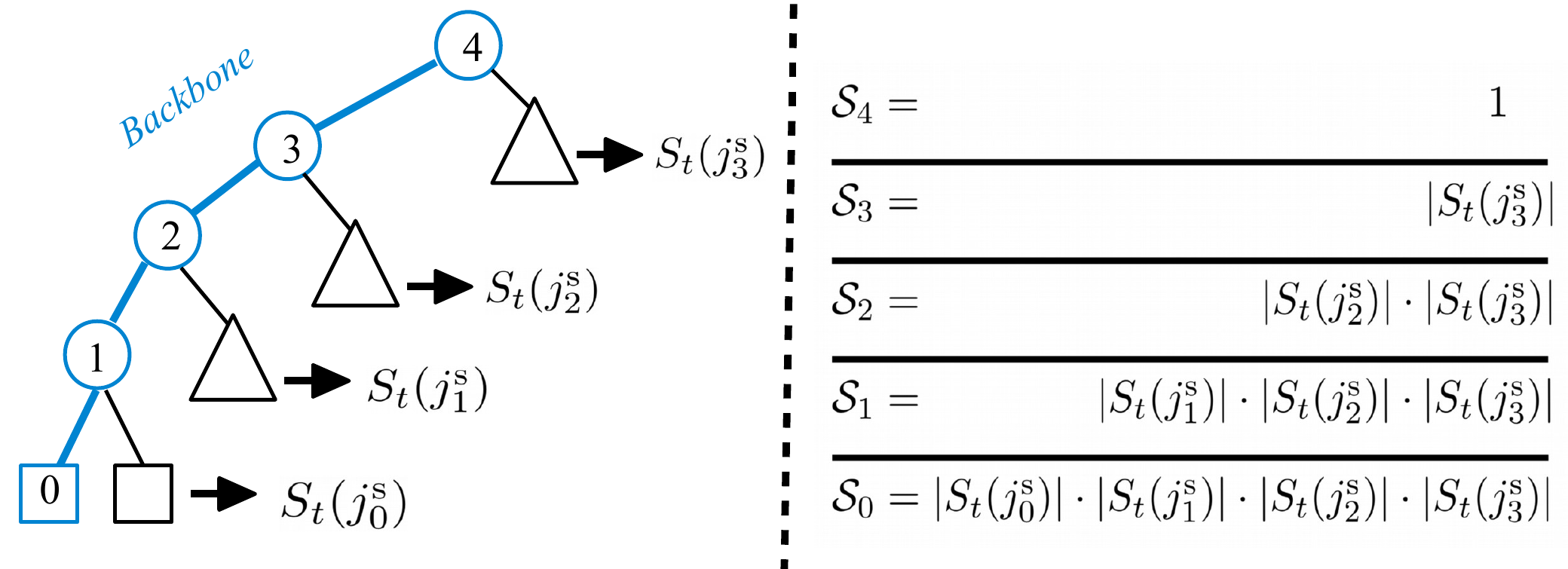}
\caption{A backbone path $\pi$ selected by \ots\ and the associated quantities needed to prove the main properties of the selected node $i^*_t$. Leaves are represented by squares, subtrees by triangles. For simplicity, in this picture $\pi$ is starting from the leftmost leaf of $T'$, but it can clearly be chosen to start from any of its deepest leaves. The sum of all terms on the right of each subtree equals $|S_t(r')|$. The cornerstone of the proof is that $i^*_t$ (and hence $z^*$) corresponds to the lowest among the $h(\pi)=4$ horizontal lines depicted in this figure for which the sum of all products below the chosen line is at least half the sum of all the products above the line. Furthermore, the fact that $|S_t(j^{\mathrm{s}}_0)|=1$ guarantees that $\scS_0=\scS_1$. Combined with the fact that $\scS_{z}\ge\scS_{z+1}$ for all $z\in \{0,\ldots, h(\pi)-1\}$, this ensures that the abovementioned horizontal line always exists, and splits the sum of all $h(\pi)+1$ terms into two parts such that the smaller one is at least $\frac{1}{3}$ of the total.}\label{f:OTSfigure}
% sum of all terms represented in the picture.
\end{figure}

The proof is now concluded by contradiction. If our statement is false, then there must exist a value $z'$ such that $\scS_{z'}>\frac{2}{3}|S_t(r')|$ and $\scS_{z'+1}<\frac{1}{3}|S_t(r')|$. However, because the sequence $\langle \scS_0, \scS_1, \ldots \scS_{h(\pi)}\rangle$ is monotonically decreasing and we have $\scS_0\le\frac{|S_t(r')|}{2}$, implying $\scS_0\le\sum_{z=1}^{h(\pi)}\scS_z$, such value $z'$ cannot exist. Thus, it must exist $z$ such that 
\[
\frac{1}{3}|S_t(r')|\le\scS_{z}\le\frac{2}{3}|S_t(r')|~.
\]
Let $z^*$ be the smallest $z$ satisfying the above inequalities. 
Note that $i_t^*$ is the parent of $j^{\mathrm{s}}_{z^*}$, because of the bottom-up search on $\pi$ performed by \ots.
Exploiting again the monotonicity of the sequence $\langle \scS_0, \scS_1, \ldots \scS_{h(\pi)}\rangle$ and recalling that 
$\scS_{z^*}=\prod_{k={z^*}}^{h(\pi)-1} |S_t(j^{\mathrm{s}}_{k})|$,
we conclude that 
\[
\frac{1}{3}|S_t(r')|\le\left|S_{t}^{y(i_t^*)=0}(r')\right|\le\frac{2}{3}|S_t(r')|~.
\] 
Since $\left|S_{t}^{y(i_t^*)=1}(r')\right|+\left|S_{t}^{y(i_t^*)=0}(r')\right|=\left|S_{t}(r')\right|$, we must also have 
\[
\frac{1}{3}|S_t(r')|\le\left|S_{t}^{y(i_t^*)=1}(r')\right|\le\frac{2}{3}|S_t(r')|,
\] 
thereby concluding the proof.
\end{proof}

From the above proof, one can see that it is indeed necessary that $\pi$ is a backbone path, since the proof hinges on the fact that $|S_t(j^{\mathrm{s}}_0)|=1$. 
In fact, if $|S_t(j^{\mathrm{s}}_0)|$ is larger than $2\sum_{z=1}^{h(\pi)}|\scS_z|$, that is larger than $\frac{2}{3}|S_t(r')|$ (which may happen if $\pi$ is not a backbone path), we would not have $\max_z \scS_z\le \frac{|S_t(r')|}{2}$, hence $\scS_z$ would not be guaranteed to be at least $\frac{2}{3}|S_t(r')|$ for all $z$.

\noindent{\bf Proof of Theorem \ref{t:ots}}\\
\begin{proof}
By Lemma~\ref{l:ots}, we immediately see that \ots\ finds $c^*$ through $\scO(\log N)$ queries. This is because $|S_{t+1}(r)|\le\frac{2}{3}|S_{t}(r)|$ for all time steps $t$, implying by induction that the total number of queries is upper bounded by $\log_{3/2} N = \scO(\log N)$.  

We now sketch an implementation of \ots\ which requires $\scO(n+h\log N)$ time and $\scO(n)$ space.
 
In a preliminary phase, we compute in a bottom-up fashion the values $|S_0(i)|$ for all nodes $i\in V(T)$. This requires $\scO(n)$. Thereafter, we perform a breath-first search on $T$, and each time we visit a leaf of $T$, we insert a pointer to it in a an array $A$ in a sequential way. Thus, the $j$-th record of $A$ will contain a reference to the $j$-th leaf found during this visit, which entails that the leaves referred by the pointers of $A$ are sorted in ascending order of depth. 

We recall that in the noiseless setting, each time the label of a node is revealed and is equal to $1$ (to $0$), also the labels of its descendants (ancestors) are indirectly revealed, because they are known to be equal to $1$ (to $0$). The total time \ots\ takes for assigning all indirectly revealed labels is clearly $\scO(n)$. 
Each time \ots\ needs to find a backbone of a tree in the current forest $F$, we look for the largest index $j$ for which the record $A[j]$ does not point to a leaf whose parent label has not been revealed yet. Observe that, at any time $t$, the deepest leaf $\ell \in L(T)$ satisfying this property must be the terminal node of a backbone path of a tree in $F$. 
%it belongs to. 
Furthermore, the highest node of such backbone is either $r(T)$ or the lowest ancestor of $\ell$ whose label has not been revealed yet, and can therefore be found in $\scO(h)$ time. 

In order to accomplish this leaf search operation, we simply maintain over time an index that scans $A$ from $A[n]$ to $A[1]$, 
%and can only be decreased by $1$, 
looking for a leaf satisfying the above property. The total time \ots\ uses for scanning $A$ is again linear in $n$. Finally, for each query, \ots\ traverses bottom-up a backbone $\pi$, exploiting the information previously stored to find $i^*_t$, and updates it after $y(i^*_t)$ is revealed. Note that only the information of the nodes in $\pi$ has to be updated. In fact, the disclosure of the label of any node $i\in V(T)$ cannot affect the values of $S_t(j)$ for all nodes $j\in V(T)$ that are {\em not} ancestors of $i$. Besides, we are free to disregard the descendants of $i$ since they will simply be indirectly labeled (by $1$). 

Overall, the total time required by this implementation of \ots\ is the sum of $\scO(n)$ and $\scO(h)$ times the total number of queries the algorithm makes, which results in the claimed $\scO(n+h\log N)$ upper bound. The claim on the memory requirement immediately follows from the above description.
\end{proof}

\subsection{Proof of the lower bound in Theorem \ref{t:LBnoiselessExistsP}}\label{as:lower}

\begin{proof}
Let $T'$ be the subtree of $T$ constructed by visiting $T$ from its root (for instance by a breadth-first or a depth-first visit), and such that $|L(T')|= B$. Note that the construction of $T'$ satisfying this constraint is always possible because the maximum cardinality of $L(T')$ is equal to $|L_{\mathrm{s}}(T)|$ (which is also equal to $\max_{c^*} \kt$). For each leaf $\ell\in L(T')$, consider all cuts $c^*$ that can be generated by cutting either the edge connecting $\ell$ with its parent or the two edges connecting $\ell$ with its children. The total number of such cuts is $2^{|L(T')|}= 2^B$. We set the prior $\Pr(\cdot)$ to be uniform over these $2^B$-many cuts. Hence, for each leaf $\ell\in L(T')$, the probability (w.r.t. $\Pr(\cdot)$) that $c^*$ cuts the edge connecting $\ell$ with its parent is 1/2, and so is the probability that $c^*$ cuts the two edges connecting $\ell$ with its children. 

Now, observe that, by construction,
%\footnote{\textcolor{NavyBlue}{Note that, within this context, $\kt$ is a random variable.}}, 
we have $B\le\kt\le 2B$ for {\em all} such cuts $c^*$ and, as a consequence, $B\le\E[\kt]\le 2B$, the expectation being over $\Pr(\cdot)$.
Since for each leaf of $T'$ any (possibly randomized) active learning algorithm $A$ has to make $\tfrac12$ mistake in expectation (over $\Pr(\cdot)$ and its internal randomization), we conclude that $B/2$ queries are always necessary to find $c^*$, %while $B\le\E[\kt]\le 2B$, 
as claimed. 
\end{proof}

%Interestingly enough, Theorem~\ref{t:LBnoiselessExistsP} implies that, unlike for the uniform prior case, in general there exist input trees $T$ and (non-uniform) prior distributions $\Pr(\cdot)$ such that the total number of queries necessary to find $c^*$ exactly is {\em not} equal to $\scO\left(\log \frac{1}{\Pr(c^*)}\right)$ for {\em all} cuts $c^*$. In fact, from 
%Theorem~\ref{t:LBnoiselessExistsP} it immediately follows that for a wide class of trees $T$ we can have $\kt=\omega(1)$, while there exists a cut $c^*$ such that $\left\lceil\log \frac{1}{\Pr(c^*)}\right\rceil=1$ because $\Pr(c^*)$ can be set arbitrarely close to $1$.

\subsection{Weighted Dichotomic Path (\wdp)}\label{as:wdp}

\begin{algorithm}[!h]
    \SetKwInOut{Input}{\scriptsize{$\triangleright$ INPUT}}
    \SetKwInOut{Output}{\scriptsize{$\triangleright$ OUTPUT}}
    \Input{~$T$,~~$\{\Pr(i), i\in V\}$.}
    \Output{~$\hCC=\CC^*$.}
    {\bf Init:} 
    \begin{itemize}
    \vspace{-0.44cm}
    \item $\hCC\gets \emptyset$;~~\comm{$\hCC$ contains all the clusters of $\CC(c^*)$ found so far}
    \item $F\gets\{T\}$;~~\comm{Forest of maximal subtrees $T'$ of $T$}
    \item $y(\pa(r))\gets 0$;~~\comm{Dummy node $\pa(r)$}\\
    \item \For{$i\in V$}{\vspace{-0.45cm}\qquad\qquad\qquad\qquad $q(i)\gets \Pr(i)\cdot\prod_{j\in\pi(\pa(i),r(T))}(1-\Pr(j))$;}
    \end{itemize}
       \vspace{0.2cm}
    \comm{ --- Path with maximum entropy --- }\\
    \While{$F\neq\emptyset$}{
    Let $L(F)$ be the set of all leaves of $T$ belonging to the subtrees in $F$.\\
    Let $R(F)$ be the set of all roots of the subtrees in $F$.\\
    $\pi(\ell, r')\gets\arg\max_{\ell\in L(F), r'\in R(F)}{H(\pi(\ell, r'))}$;\\
    $\scT=\{T'\in F: \ell, r'\in V(T')\}$;\\%~~\comm{$\scT$ only contains the subtrees $T'$ s.t. $\ell, r'\in V(T')$}\\
    \vspace{0.2cm}
    \comm{ --- Binary search on path $\pi(\ell, r')$ --- }\\
    $u\gets\ell$;~~~~$v\gets r'$;\\
    \While{$u\neq v$}{
Let $\langle i_0=u, i_1, \ldots, i_{h-1}, i_h=v \rangle$ be the sequence of nodes lying on $\pi(u, v)$ in descending order of depth.\\
	$Q\gets\sum_{k\in\{0,1,\ldots,h-1\}}q(i_k)$;\\
	$k^*=\arg\min_{k\in\{0, 1, \ldots, h-1\}}\left|\frac{Q}{2}-\sum_{j\in\{ i_0, \ldots, i_k \}} q(j) \right|$;\\
    $i^*\gets\pa(i_{k^*})$;\\
    Query $y(i^*)$;\\
    \eIf{$y(i^*)=0$}
    {%Set $\Pr(i)$, $y(i)$ and $q(i)$ equal to $0$ for all $i\in \pi(i^*, v)$;\\
    %Update $q(i')$ for all $i'\in \pi(u, i_{k^*})$;\\
    $v\gets i_{k^*}$;%~~\comm{Child of the lowest node on $\pi(\ell, r')$ whose label is currently known to be equal to $0$}
     }
    {%Update $q(i^*)$;\\
    $u\gets i^*$;%~~\comm{Highest node on $\pi(\ell, r')$ whose label is currently known to be equal to $1$}
    }    
    }
    Set $\Pr(i) = y(i) = 1$ for all descendants $i$ of $u$;\\
    Set $\Pr(i) = y(i) = 0$ for all ancestors $i\neq u$ of $u$;\\
    $q(u)\gets 1$;~~~~~~Set $q(i) = 0$ for all descendants and ancestors $i\neq u$ of $u$;\\
    Update $\Pr(i)$ and $q(i)$ for all descendants of all $i\in V(\pi(\pa(u),r'))$ such that $i\neq r'$;\\
    $\hCC\gets \hCC\cup \{L(u)\}$;~~\comm{$L(u)$ is a cluster of $\CC^*$}\\
    \vspace{0.2cm}
    \comm{ --- Update F --- }\\
    $F\gets F\setminus\scT$;~~\comm{Remove from F the subtree containing $\pi(\ell, r')$}\\
    $j\gets\pa(u)$; ~\comm{Lowest node in $\pi(\ell, \pa(r'))$ with label known to be $0$}\\
    \While{$j\in V(\pi(\ell,r'))$}{
    %$V'\gets V\setminus V(\pi(\ell,r'))$;\\
    Let $j_c$ be the child of $j$ that is not in $\pi(\ell,r')$.\\
    \eIf{$j_c\in L$}
    {$\hCC\gets \hCC\cup \{j_c\}$;~~\comm{Add to $\hCC$ a singleton cluster}}{$F\gets F\cup \{T(j_c)\}$;~~\comm{$j_c$ is the root of a subtree that will be processed later}\\
    Update $q(i)$ for all $i\in V(T(j_c))$;\\}
    $j\gets \pa(j)$;~~\comm{j is a node whose label is known to be $0$}\\
    }
    }
	\Return{$\hCC$~.}
    \caption{WDP (Weighted Dichotomic Path) \label{a:wdp}}	 
\end{algorithm}

In Algorithm \ref{a:wdp} we give the pseudocode of \wdp.
At each round, \wdp\ finds the path 
%$\pi(\ell, r')$ 
whose entropy is maximized over all bottom-up paths $\pi(\ell, r')$, with $\ell \in L$ and $r' = r(T')$, where $T'$ is the subtree in $F$ containing $\ell$. Ties are broken arbitrarily.
%backbone paths.
% (with ties broken arbitrarily).
%over all the bottom-up paths of all trees in $F$ . 
%Let $T' \in F$ be the subtree containing the selected path, so that $r' = r(T')$. 
%Note that, by the definition of $F$ and \enquote{entropy of a path}, $\ell$ must be a leaf of $T'$ and $r'$ must be its root, because \wdp\ selectes the bottom-up path having maximum entropy. 
%Note that, by the definition of $F$ and \enquote{entropy of a path}, $\ell$ must be a leaf of $T'$ and $r'$ must be its root, because \wdp\ selectes the bottom-up path having maximum entropy. 
\wdp\ performs a binary search on such $\pi(\ell,r')$ to find the edge of $T'$ which is cut by $c^*$, %and lies on this path. 
taking into account the current values of $q(i)$ over that path.
Specifically, let $\langle i_0=\ell, i_1, \ldots, i_{h-1}, i_h=r' \rangle$ be the sequence of nodes in $\pi(\ell, r')$ in descending order of depth. 
\wdp~finds an index $k^*$ that corresponds to the middle point in $\pi(\ell, r')$, taking into account the current values of $q(i)$ over that path.
%such that 
%\(
%k^*=\arg\min_{k\in\{0, 1, \ldots, h-1\}}\left|\frac{1}{2}-\sum_{j\in\{ i_0, \ldots, i_k \}}q(j)\right|~
%\)
%(again, ties are broken arbitrarily).
%
Let $i^* = \pa(i_{k^*})$. \wdp\ queries the label of $i^*$: If $y(i^*)=0$, \wdp\ continues the binary search on $\pi\left(\ell, i_{k^*}\right)$;
% where $i^*_c$ is the child of $i^*$ on $\pi(\ell, r')$; 
%
if instead $y(i^*)=1$, the binary search continues on $\pi\left(i^*, r'\right)$, and so on. %After each label revelation, the algorithm updates the probabilities $\Pr(i)$ of all nodes $i$ belonging to the selected path. 
During the binary search, whenever \wdp\ finds a node $u\in V(\pi(\ell, r'))$ with queried labels $y(u) = 0$ and $y(\pa(u)) = 1$,
% are both revealed, and are equal, to $0$ and $1$ respectively, 
then the edge of $\pi(\ell, r')$ cut by $c^*$  has been found, and the binary search on this backbone path terminates. In the special case where $y(r') = 1$, the binary search also ends, and we know that all nodes in $L(r')$ form a cluster of $\CC(c^*)$. 
Once a binary search terminates, \wdp\ updates $F$ and the probabilities $\Pr(i)$ at all nodes $i$ in the subtrees of $F$, so as to reflect the new knowledge gathered by the queried labels. 

Below, we prove \wdp's query complexity.
%Throughout, $q(i)$ will be meant w.r.t. the {\em posterior} distribution $\Pr(\cdot)$ at the {\em current} time under consideration. This will cause no ambiguity, even though we do not equip $\Pr(\cdot)$ and $q(i)$ with an explicit time subscript.
%
%We need to introduce further notation. Given tree $T$, we denote by $\scP_{>0}$ the set of all prior distributions $\Pr(\cdot)$ 
%such that for each cut $c$ of $T$ we have $\Pr(c)>0$.
The proofs are split into a series of lemmas.
\begin{lemma}\label{l:WDPentropyDifference}
Given tree $T$ with set of leaves $L$, any prior $\Pr(\cdot)\in\scP_{>0}$ over $c^*$, and any $c^* \sim \Pr(\cdot)$, let $j_0$ be a node of $\AB(c^*)$, having as children a leaf $\ell\in L$ and an internal node $j'$ of $T$ (see Figure~\ref{f:WDPproofs}, left). Then, during its execution, \wdp\ will never select the bottom-up path starting from $\ell$. 
\end{lemma}
\begin{proof}
%
%First of all we show that in case (3) and (5), the cut edge connecting $\ell$ with and $i_0$ cannot be part of any path in $\Pi$, and it is therefore indirectly discovered by \wdp. 
%
Let $T_0$ be the tree made up of all nodes of $\AB(c^*)$, and consider any given round with $q(i)$ in (\ref{e:qi}) defined by the current {\em posterior} distribution maintained by the algorithm. We first show that, for all ancestors $a$ of $j_0$, path $\pi(\ell, a)$ cannot be selected by \wdp, because its entropy\footnote
{
Here, we are defining the entropy of a path $\pi$ as $-\sum_{v \in V(\pi)} q(v)\log_2 q(v)$, even for paths $\pi$ for which $\sum_{v \in V(\pi)} q(v) < 1$.
} 
$H(\pi(\ell, a))$ will always be strictly smaller than $H(\pi(\ell', a))$ for {\em all} leaves $\ell' \in L(j')$. To this effect, we can write
\begin{align}
H(\pi(\ell,a))-H(\pi(\ell',a))&=\left(-q(\ell)\log_2 q(\ell) - \sum_{u\in \pi(j_0, a)} q(u)\log_2 q(u)\right)\nonumber\\
&\qquad\ \ -\Biggl(H(\pi(\ell',j') - \sum_{u\in \pi(j_0, a)}q(u)\log_2 q(u)\Biggr)\nonumber\\
&=-q(\ell)\log_2 q(\ell)+\sum_{v\in\pi(\ell',j')}q(v)\log_2 q(v)~.\label{e:difference}
\end{align}
Now, since
\[
\sum_{u\in\pi(\ell',j')}q(u) + \sum_{v\in\pi(j_0,a)}q(v) 
= 
\sum_{u\in\pi(\ell',a)}q(u)
= \sum_{u\in\pi(\ell,a)}q(u) = q(\ell)+\sum_{u\in\pi(j_0,a)}q(u)~,
\]
we have $q(\ell)=\sum_{v\in\pi(\ell',j')}q(v)$. 

Consider the function $f(x)=-x\log_2 x$, for $x\in [0,1]$. For all $x,x_1,x_2\in(0,1)$ such that $x_1+x_2=x$, it is easy to verify that we have $f(x)<f(x_1)+f(x_2)$. More generally, for all $x, x_1, x_2, \ldots, x_m\in(0,1)$ with $\sum_{i=1}^m x_i = x$, one can show that $f(x)<\sum_{i=1}^m f(x_i)$. 
%To see why, it is sufficient to recursively apply $(m-1)$-many times the above inequality $f(x)<f(x_1)+f(x_2)$ (which holds for only two values, namely $x_1$ and $x_2$), by creating a binary tree whose leaves corresponds to $x_1, x_2, \ldots, x_m$ and the internal nodes $i$ corresponds to an application of the inequality for the sum of all leaves descendant from the left child of $i$ and the ones descendant from its right child. 
Since $|V(\pi(\ell',j'))|\ge 2$ (holding because $j'\not\in L$ implies $\ell'\neq j'$), the above inequality on $f(\cdot)$ allows us to write
% (that holds for at least two values in $(0,1)$) and write
\[
-q(\ell)\log_2 q(\ell)<-\sum_{v\in\pi(\ell',j')}q(v)\log_2 q(v)~,
\]
i.e., (\ref{e:difference})  < 0. Notice that the assumption $\Pr(\cdot)\in\scP_{>0}$ implies $q(v)>0$ at any stage of the execution of \wdp\ where node $v$ has an unrevealed label. This is because, after any binary search on a path selected by \wdp, for all $v$ belonging to any tree in $F$, in the update phase each value $q(v)$ is multiplied by a strictly positive value. This ensures that we can use the above inequality about $f(\cdot)$, as its argument will always lie in the open interval $(0,1)$.

%, as assumed above. Note also that if one its argument was equal to $0$, the result would not hold anymore, because for all $x,x_1,x_2\in[0,1)$ such
%that $x_1+x_2=x$, we have $f(x)$ is not {\em strictly} smaller than $f(x_1)+f(x_2)$.

The inequality in (\ref{e:difference}) implies that there always exists a leaf $\ell'$ of $T(j')$ such that \wdp\ selects the path connecting $\ell'$ with the root of the tree containing $\ell$ in the current forest $F$. This selection entails the disclosure of either cut edge $(j',j_0)$ (if $j_0\in L(T_0)$), or a cut edge in $T(j')$ (if $j_0\not\in L(T_0)$), which in turn implies that the labels of $i_0$ and all its ancestors will be disclosed to the algorithm to be equal to $0$, thereby indirectly revealing also cut edge $(\ell, i_0)$. Since $F$ contains only trees whose height is larger than $1$, after this cut edge disclosure the tree made up of leaf $\ell$ alone cannot be part of $F$, thus preventing \wdp's selection of a path starting from $\ell$.
\end{proof}

\begin{figure}\
\center
\includegraphics[width=9cm]{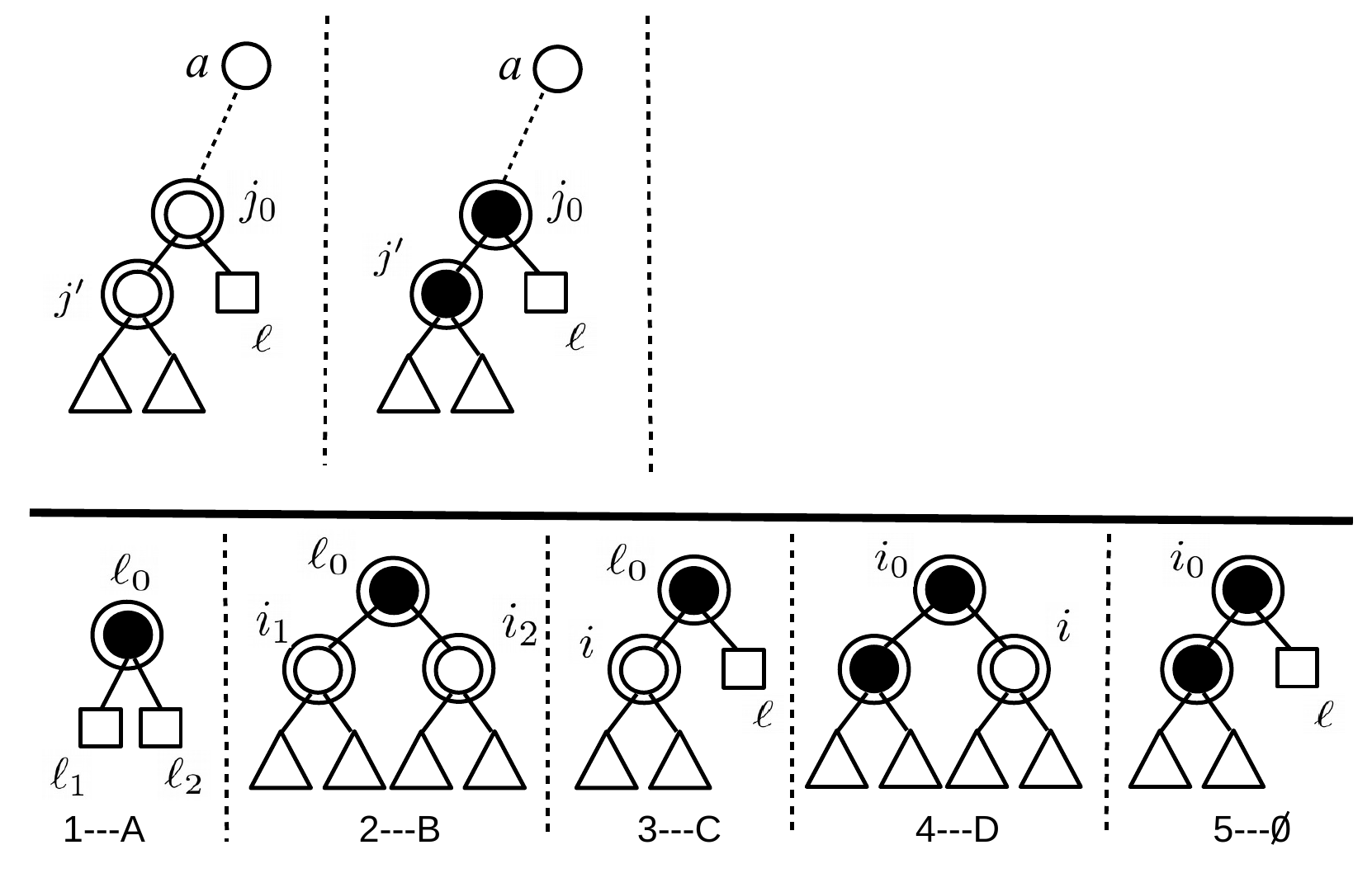}
\caption{Illustration of all possible cases of Lemma~\ref{l:WDPentropyDifference} and Lemma~\ref{l:WDPnumberPaths}. Nodes belonging to $T_0$ (see main text) are black, all remaining nodes are white. Leaves and subtrees of $T$ are represented by squares and triangles, respectively. Each node of $T'_{c^*}$ is enclosed in a circle. {\bf Above:} The two possible cases illustrating Lemma~\ref{l:WDPentropyDifference}, that is, $j'\not\in V(T_0)$ on the left, and $j'\in V(T_0)$ on the right. {\bf Below:} The five cases described in Lemma~\ref{l:WDPnumberPaths}.}\label{f:WDPproofs}
% sum of all terms represented in the picture.
\end{figure}

\begin{lemma}\label{l:WDPnumberPaths}
For any input tree $T$ and any cut $c^*$ with $\Pr(\cdot)\in\scP_{>0}$, the number of paths selected by \wdp\ before stopping is $\kt$.
\end{lemma}
\begin{proof}
If $c^*$ has only one cluster the statement is clearly true, since the binary search performed by \wdp\ on the first selected path reveals that $y(r)=1$ (hence $y(v) = 1$ for all $v\in V$). We then continute by assuming $y(r) = 0$, so that $c^*$ has least two clusters. 

Let $\Pi$ be the set of all paths selected by \wdp\, during the course of its execution. The binary search perfomed by \wdp\ on {\em each} of such paths, discloses exactly {\em one} edge of $c^*$. Let $c^*_{\wdp}$ be the set containing all these cut edges, and $c^*_0$ be the set of the remaining cut edges of $c^*$.
We show that, for any $T$ and any cut $c^*$ of $T$, $|c^*_{\wdp}|=\kt$, while all edges in $c^*_0$ are {\em indirectly} disclosed by \wdp, although none of them belongs to paths in $\Pi$. 

Let $T_0$ be the subtree of $T$ made up of all nodes in $\AB(c^*)$. The edges of $c^*$ can be partitioned into the five disjoint sets $S_1, \ldots, S_5$ (see Figure~\ref{f:WDPproofs} for reference):
\begin{itemize}
\item [$S_1$:] The set of all {\em pairs} of edges connecting a leaf $\ell_0$ of $T_0$ to two sibling leaves $\ell_1$ and $\ell_2$ of $T$ (Figure \ref{f:WDPproofs}, below, 1); 
\item [$S_2$:] The set of all {\em pairs} of edges connecting a leaf $\ell_0$ of $T_0$ to two sibling internal nodes $i_1$ and $i_2$ of $T$ (Figure \ref{f:WDPproofs}, below, 2);  
\item [$S_3$:] The set of all {\em pairs} of edges connecting a leaf $\ell_0$ of $T_0$ to a leaf $\ell\in L$ and an internal node $i$ of $T$ (Figure \ref{f:WDPproofs}, below, 3); 
\item [$S_4$:] The set of all edges connecting an internal node $i_0$ of $T_0$ to an internal node $i$ of $T$, so that the sibling node of $i$ belongs to $V(T_0)$ (Figure \ref{f:WDPproofs}, below, 4); 
\item [$S_5$:] The set of all edges connecting an internal node $i_0$ of $T_0$ to a leaf $\ell$ of $T$, so that the sibling node of $\ell$ belongs to $V(T_0)$ (Figure \ref{f:WDPproofs}, below, 5).
\end{itemize}
Recall that $T'_{c^*}$ is the subtree of $T$ whose nodes are $(\AB(c^*) \cup \LB(c^*)) \setminus L$, and that $\kt$ is the number of its leaves. The leaves of $T'_{c^*}$ can be partitioned into the following four sets $A$, $B$, $C$, and $D$ (see again Figure~\ref{f:WDPproofs} for reference): 
\begin{itemize}
\item [$A$:] The set of all leaves of $T'_{c^*}$ that are also leaves of $T_0$, i.e., that belong to $\AB(c^*)$;
\item [$B$:] The set of all {\em sibling} leaves of $T'_{c^*}$ that are also (sibling) internal nodes of $T$;
\item [$C$:] The set of all leaves of $T'_{c^*}$ that are also internal nodes of $T$ such that their sibling node is a leaf of $T$;
\item [$D$:] The set of all leaves of $T'_{c^*}$ that are also internal nodes of $T$ such that their sibling node belongs to $T_0$.
\end{itemize}
We will not show a one-to-one mapping between $L(T'_{c^*})$ and the cut edges of $c^*_{\wdp}$ covering all possible cases.
\begin{itemize}
\item [$S_1 \leftrightarrow A$:] For each pairs of cut edges in $S_1$), \wdp\ clearly selects a path starting from either $\ell_1$ or $\ell_2$, which will indirectly disclose the cut edge incident to the sibling leaf ($\ell_2$ or $\ell_1$, respectively). $S_1$ is therefore about all leaves of set $A$.
\item [$S_2 \leftrightarrow B$:] For each pairs of cut edges in $S_2$, \wdp\ selects two paths, one per cut edge. Each of these two paths clearly contains one of these two cut edges, and corresponds to all leaves of $T'_{c^*}$ that are also leaves of $T$. Hence we are covering all leaves of set $B$.
\item [$S_3 \leftrightarrow C$:] For the edges in $S_3$, \wdp\ selects only one path, starting from a leaf of $T(i)$. This path clearly contains edge $(i,\ell_0)$, and covers all leaves of set $C$. Observe that, by Lemma~\ref{l:WDPentropyDifference}, edge $(\ell_0, \ell)$ is always indirectly revealed and never contained in a path selected by \wdp.
\item [$S_4 \leftrightarrow D$:] For the edges in $S_4$, whenever \wdp\ selects a path starting from a leaf of $T(i)$, all the nodes in $V(T(i))$ are indirectly labeled $1$, and from that point on, they will not be included in a tree in $F$. This path clearly contains edge $(i_0, i)$, hence we are covering all leaves of set $D$. 
\item [$S_5 \leftrightarrow\emptyset$:] Finally, Lemma~\ref{l:WDPentropyDifference} ensures that all cut edges in $S_5$ are indirectly disclosed whenever \wdp\ selects a path starting from a leaf belonging to $T(i'_0)$, where $i'_0$ is the sibling node of $\ell$. Hence this case is ruled out by Lemma~\ref{l:WDPentropyDifference}, and does not correspond to any leaf.
\end{itemize}
From the above, we conclude that the number of paths selected by $\wdp$ is always equal to $\kt$, as claimed.
\end{proof}
%PROOF END

The next lemma provides an entropic bound on the (condionally) expected number of queries \wdp\, makes on a given path. Notice that the posterior distribution maintained by \wdp\, never changes {\em during} each binary search, but only between a binary search and the next. Consider then $q(i)$ defined in (\ref{e:qi}) at the beginning of a given binary search in terms of the current posterior distribution, and let $\pi$ be the selected path, after having observed the labels that generated the current posterior.
%
%Let $\Pi$ be the set of all paths selected by \wdp. Given any $\pi\in\Pi$, we denote by $Q(\pi)$ the number of queries made on $\pi$ by \wdp.
%
\begin{lemma}\label{l:piQueries}
Let $\pi$ be any path selected by \wdp\ during the course of its execution, and $\{q(i)\}$ be the current distribution (\ref{e:qi}) at the time $\pi$ is selected. Then the expected number of queries \wdp\ makes on $\pi$, conditioned on past revealed labels, is $\scO\left(\left\lceil H(\pi)\right\rceil\right)=\scO\left(\log (|V(\pi)|)\right)$~. Here, both the conditional expectation and $H(\pi)$ are defined i.t.o. $\{q(i)\}$.
\end{lemma}
\begin{proof}
%Denote by $\Pi$ the set of all paths selected by \wdp\, during the course of its execution. 
Let $\pi$ be the currently selected path,
%For any path $\pi\in\Pi$, we 
and denote by $E_{\pa}(\pi)$ the set made up of the edges in $\pi$ along with the edge connecting the top node of $\pi$ to its parent (recall that in the special case where $r$ is a terminal node of $\pi$, we can view $r$ as the child of a dummy \enquote{super-root}). 
The binary search performed on $\pi$ guarantees that the number of queries $Q(\pi,(u,v))$ made by \wdp\ to find a cut edge $(u,v)$ lying on $\pi$ can be quantified as follows:
\[
\left\lceil\log_2\left(\frac{\sum_{(u',v')\in E_{\pa}(\pi)}\Pr((u',v')\in c^*)}{\Pr((u,v)\in c^*)}\right)\right\rceil=
\left\lceil\log_2\left(\frac{1}{\Pr((u,v)\in c^*)}\right)\right\rceil~,
\]
where the probabilities above are defined w.r.t. the posterior distribution at the beginning of the binary search.
The expected number of queries made on $\pi$, conditioned on past labels can thus be bounded as
\begin{align*}
%\E Q(\pi)&=\sum_{(u',v')\in E(\pi)} \Pr((u',v')\in c^*) Q(\pi,(u',v'))\\
%&=
\sum_{(u',v')\in E_{\pa}(\pi)} \Pr((u',v')\in c^*)\left\lceil\log_2\left(\frac{1}{\Pr((u',v')\in c^*)}\right)\right\rceil
&=\sum_{u\in V(\pi)} q(i)\left\lceil\log_2\left(\frac{1}{q(i)}\right)\right\rceil\\
&=\scO\left( \left\lceil H(\pi)\right\rceil\right)\\
&=\scO\left(\log(|V(\pi)|)\right)~,
\end{align*}
as claimed.
\end{proof}

We are now ready to prove Theorem \ref{t:wdp1} and Theorem \ref{t:wdp2}.

\begin{proof}[Theorem \ref{t:wdp1}]
For given $c^*$, let $\Pi = \Pi(c^*) = \langle \pi_1, \dots, \pi_{|\Pi|}\rangle$ be the sequence of paths selected by \wdp, sorted in the temporal order of selection during \wdp's run. 
%and ${\bar \Pi}$ be the set of the remaining bottom-up paths, starting from the leaves, which have not been selected. Notice that $|\Pi \cup {\bar \Pi}| = n$. 
Also, denote by $Q(\pi_j)$ the number of queries made by \wdp\, on $\pi_j \in \Pi$. Notice that the sequence $\Pi$ is fully determined by $c^*$. 
Moreover, the paths in $\Pi$ are orderer in such a way to guarantee that $\pi_j$ contains a unique edge $(\pa(u_j),u_j)$ which $c^*$ cuts across. Then, if we denote by $\{q_j(\cdot)\}$ the value of $q(\cdot)$ at the time path $\pi_j$ is selected, it is easy to see that cut $c^*$ can be alternatively generated by sequentially generating edge $(\pa(u_1),u_1)$ according to distribution $\{q_1(\cdot)\}$ over $\pi_1$, then $(\pa(u_2),u_2)$ according to (posterior) distrubution $\{q_2(\cdot)\}$ over $\pi_2$, then $(\pa(u_3),u_3)$ according to (posterior) distrubution $\{q_3(\cdot)\}$ over $\pi_3$, and so on until $|\Pi|$ cuts have been generated. From Lemma \ref{l:WDPnumberPaths}, we have $|\Pi| = \kti(T,c^*)$.

%
%Let $\langle \pi_1, \dots, \pi_{|\Pi|}, \pi_{|\Pi|+1}, \ldots \pi_{n} \rangle$ be the temporal order of selection of the $|\Pi|$ paths, padded with the remaining $|{\bar \Pi}| = n - |\Pi|$ paths (sorted in any relative order).
%the concatenation of the sequence of selected paths by \wdp\ and the sequence (in an arbitrary order) of all {\em single-ton} paths where each of them contains a leaf that does not belong to any path in $\Pi$. By this definition, we have $|V(\pi')|=1$ and $Q(\pi')=0$ for all single-ton paths $\pi'$, because \wdp\ does not make any query for the paths that do not belong to $\Pi$. 
%Now, for all paths 
%$\pi\in\Pi$, the expected number of queries made by \wdp\ is upper bounded by $\log(h)$.
%
Let us then denote by $\E[\cdot]$ the expectation w.r.t. the {\em prior} distribution, and by $\E_j[\cdot]$ be the conditional expectation $\E[\cdot\,|\, (\pa(u_1),u_1), (\pa(u_2),u_2), \ldots, (\pa(u_{j-1}),u_{j-1})]$. Notice that the sequence of random variables 
$\pa(u_1),u_1), (\pa(u_2),u_2), \ldots, (\pa(u_{j-1}),u_{j-1})$ fully determines the posterior distribution $\{q_j(\cdot)\}$ before the selection of the $j$-th path $\pi_j$, and so, $\pi_j$ itself. 
%Let also $\{\cdot\}$ be the indicator function of the predicate at argument. 
This way of viewing $c^*$ makes $\kti=\kti(T,c^*)$ a (finite) stopping time w.r.t. the sequence of random variables $\pa(u_1),u_1), (\pa(u_2),u_2), \ldots,$, in that $\{\kti \geq j\}$ is determined by $(\pa(u_1),u_1), (\pa(u_2),u_2), \ldots, (\pa(u_{j-1}),u_{j-1})$. This allows us to write
\begin{align}
\E \left[\sum_{j=1}^{\kti} Q(\pi_j)\right]
%&= \sum_{i=1}^\infty \E\left[\sum_{j=1}^{i}Q(\pi_j)\{\kti=i\} \right]\notag\\
&= \sum_{i=1}^n \sum_{j=1}^i \E\left[Q(\pi_j)\{\kti=i\} \right]\notag\\
&= \sum_{j=1}^n \E\left[Q(\pi_j)\{\kti \geq j\} \right]\notag\\
&= \sum_{j=1}^n \E\left[\{\kti \geq j\}\E_j [Q(\pi_j)] \right]\qquad  \qquad {\mbox{(since $\kti$ is a stopping time)}}\notag \\
&= \sum_{j=1}^n \sum_{i=j}^n\E\left[\{\kti =i \}\E_j [Q(\pi_j)] \right]\notag\\
&= \sum_{i=1}^n \sum_{j=1}^i \E\left[\{\kti =i \}\E_j [Q(\pi_j)] \right]\notag\\
&= \sum_{i=1}^n \E\left[\{\kti =i \} \sum_{j=1}^{\kti} \E_j [Q(\pi_j)] \right]\notag\\
&= \E\left[\sum_{j=1}^{\kti} \E_j [Q(\pi_j)] \right]\notag\\
&= \scO\left(\E\left[\sum_{j=1}^{\kti} \lceil H_j(\pi_j)\rceil \right]\right)\label{e:mainineq}\qquad {\mbox{(by Lemma \ref{l:piQueries}, where $H_j(\cdot)$ is w.r.t. $\{q_j(\cdot)\}$)}}\\
&= \scO\left(\E[\kti]\log h \right)~,\notag
\end{align}
as claimed
%Now, by Lemma \ref{l:piQueries}, we have $\E_j [Q(\pi_j)] = \scO( \lceil H_j(\pi_j)\rceil ) = \scO(\log h)$, where the entropy $H_j(\cdot)$ is computed w.r.t. distribution $\{q_j(\cdot)\}$. Combined with Eq. (\ref{e:mainineq}) gives the claimed bound.
%
\end{proof}

\vspace{-0.05in}
A slightly more involved guarantee for \wdp\, is given by the following theorem, where the query complexity depends in a more detailed way on interplay between $T$ and the prior $\Pr(\cdot)$. Given any bottom-up path $\pi$ in $T$, we denote by $\hn(\pi)$ the {\em normalized} entropy of $\pi$, defined as
%\footnote
%{
%Note that for each bottom-up path $\pi$ starting from a leaf of $T$, we clearly have $H(\pi)\le\hn(\pi)$.
%}
\(
\hn(\pi)=-\sum_{i\in V(\pi)} {\widehat q}(i)\log({\widehat q}(i))~,
\)
where ${\widehat q}(i) = q(i)/\sum_{i\in V(\pi)} q(i)$, and $q(i)$ is defined according to the prior distribution $\Pr(\cdot)$, as in (\ref{e:qi}). Notice that we may have $\sum_{i\in V(\pi)} q(i) < 1$.
%Notice that when $\pi$ starts from a leaf of $T$, $\hn(\pi)$ is the entropy $H(\pi)$ of $\pi$.%, when this path is selected by \wdp\ (if this event occurs).
%
Further, denote by $\DDD$ the family of all sets $\Pi$ of all vertex-disjoint bottom-up paths starting from $T$'s leaves such that the top terminal node of each path $\pi\in\Pi$ is either the root $r$ of $T$ or a node of another path of $\Pi$. 
%We now provide an upper bound of the expected number of queries made by \wdp\ that carefully takes into account the information provided by $\Pr(\cdot)$. 
The upper bound in the following theorem is {\em never} worse than the upper bound in Theorem~\ref{t:wdp1}.
\begin{theorem}\label{t:wdp2}
In the noiseless realizable setting, for any tree $T$, any prior distribution $\Pr(\cdot)$ over $c^*$ such that $\Pr(\cdot) \in \scP_{>0}$, the expected number of queries made by \wdp\ to find $c^*$ is
\(
\scO\left(\max_{\Pi\in\DDD}\sum_{j=1}^{m(\Pi)} \lceil\hn(\pi_{i_j})\rceil\right)~,
\) 
where $m(\Pi)=\min\left\{\left\lceil \E\kt\right\rceil,|\Pi|\right\}$, and paths $\pi_{i_1}, \pi_{i_2}, \ldots$ in $\Pi \in \DDD$ are sorted in non-increasing value of normalized entropy $\hn(\cdot)$. In the above, the expectations is again over $\Pr(\cdot)$.
\end{theorem}
As an application of the above result, consider that oftentimes a linkage function generating $T$ also tags each internal node $i$ with a coherence level $\alpha_i$ of $T(i)$, which is typically increasing as we move downwards from root to leaves. A common situation in hierarchical clustering is then to figure out the ``right" level of granularity of the flat clustering we are looking for through the definition of bands of nodes (i.e., bands of clusters) of similar coherence. This may be encoded through a prior $\Pr(\cdot)$ that uniformly spreads $(1-\epsilon)/b$ probability mass over $b$-many edge-disjoint cuts of $T$, for $b \ll h$, and an arbitrarily small $\epsilon >0$, and the remaining mass $\epsilon$ over all remaining cuts (this is needed to comply with the condition $\Pr(\cdot) \in \scP_{>0}$). As we said in the main body of the paper, Theorem \ref{t:wdp2} gives a bound of the form\, $\E\kt\log b$\, as opposed to the bound\, $\E\kt\log h$\, provided by Theorem \ref{t:wdp1}.

\noindent{\bf Proof of Theorem \ref{t:wdp2}}\\
\begin{proof}
Given $T$ and prior $\Pr(\cdot)$, let $\DDD_{\wdp}$ be the set made up of all sets $\Pi$ of bottom-up paths in $T$ that \wdp\ can potentially select during the course of its executions. Each set $\Pi$ is uniquely determined by $c^* \sim \Pr(\cdot)$. The family of sets $\DDD$ is clearly a superset of $\DDD_{\wdp}$.
We prove the theorem by showing that the expected number of queries made by \wdp\ is upper bounded by 
%(throughout this proof we assume that if index $i$ is greater than $|\Pi|$, then $\hn(\pi^{\downarrow}_i)$ is set to be equal to $0$ thereby providing a null contribution to the sum)
%
\begin{equation}\label{e:prebound}
\scO\left(\max_{\Pi\in\DDD_{\wdp}}\sum_{j=1}^{m(\Pi)} \left\lceil\hn(\pi_{i_j})\right\rceil\right)~,
\end{equation}
where, for any given $\Pi\in\DDD_{\wdp}$, $\pi_1, \pi_2, \dots $ is the sequence of paths of $\Pi$ in the order they are selected by \wdp\,, while $\pi_{i_1}, \pi_{i_2}, \dots $ is the same sequence rearranged in non-increasing order of $\hn(\cdot)$. 
Using the same notation as in the proof of Theorem \ref{t:wdp1}, we observe that at the time when $\pi_j$ gets selected by \wdp\, the distribution $\{q_j(\cdot)\}$ sitting along path $\pi_j$ is precisely the normalized distribution $\{{\widehat q}(\cdot)\}$ such that $\sum_{i=1}^{|V(\pi_j)|}{\widehat q}(i)=1$, so that $H_j(\pi_j) = \hn(\pi_j)$. 
Then, Eq. (\ref{e:mainineq}) combined with Lemma \ref{l:piQueries} allows us to write
\begin{align*}
\E \left[\sum_{j=1}^{\kti} Q(\pi_j)\right]
= \E\left[\scO\left(\sum_{j=1}^{\kti} \lceil \hn_j(\pi_j)\rceil \right)\right]~.
\end{align*}
In the sequel, we show how to upper bound the right-hand side of the last (in)equality by (\ref{e:prebound}). Set for brevity $\E[\kti] = \lceil\mu\rceil$. We have
\begin{align*}
\sum_{j=1}^{\kti} \lceil \hn_j(\pi_j)\rceil 
& = 
\sum_{j=1}^{\kti} \{\kti < \mu\} \lceil \hn_j(\pi_j)\rceil  + \sum_{j=1}^{\kti} \{\kti \geq \mu\} \lceil \hn_j(\pi_j)\rceil\\
&\leq 
\sum_{j=1}^{\mu} \lceil \hn_j(\pi_{i_j})\rceil  + \frac{\kti}{\mu}\,\sum_{j=1}^{\kti}  \lceil \hn_j(\pi_j)\rceil\\
&\leq 
\max_{\Pi\in\DDD_{\wdp}} \sum_{j=1}^{m(\Pi)} \lceil \hn_j(\pi_{i_j})\rceil  + \frac{\kti}{\mu}\, \max_{\Pi\in\DDD_{\wdp}} \sum_{j=1}^{m(\Pi)}  \lceil \hn_j(\pi_{i_j})\rceil\\
&=
\left(1+\frac{\kti}{\mu} \right) \max_{\Pi\in\DDD_{\wdp}} \sum_{j=1}^{m(\Pi)} \lceil \hn_j(\pi_{i_j})\rceil~,
\end{align*}
so that, taking the expectation of both sides, 
\[
\E\left[\sum_{j=1}^{\kti} \lceil \hn_j(\pi_j)\rceil\right] \leq 2\, \max_{\Pi\in\DDD_{\wdp}} \sum_{j=1}^{m(\Pi)} \lceil \hn_j(\pi_{i_j})\rceil~.
\]

This concludes the proof.
\end{proof}

\subsection{\nwdp~(Noisy Weighted Dichotomic Path)}\label{as:nwdp}

\nwdp\, is a robust variant of \wdp\ that copes with persistent noise. 
Given an internal node $i\in V\setminus L$, let $\scL(i)$ be the set of all possible queries that can be made to determine $y(i)$, i.e., the set $(\ell,\ell')\in L(\lef(i))\times L(\ri(i))$. %\wdp~can be extended by reinforcing the confidence of the answer(s) received to determine a label of a node $v\in V\setminus L$, whenever $|\scL(v)|$ is sufficiently large. We briefly describe the operations performed by \nwdp, which consists of the following two steps: 
Then, given confidence $\delta \in (0,1]$, and noise level $\lambda \in [0,1/2)$, \nwdp:
\begin{enumerate}
\item Preprocesses $T$ and prior $\Pr(\cdot)$ by setting $y(i) = 1$ for all nodes $i\in V\setminus L$ such that $|\scL(i)|<\frac{\alpha\log(n/\delta)}{(1-2\lambda)^2}$, for a suitable constant $\alpha>0$. $\Pr(\cdot)$ is also updated (all $j \in T(i)$ have $P(j) = 1$). At the end of this phase, each node in $V$ is either unlabeled or labeled with $1$. 
%We denote by $\Pr_{\lambda}(\cdot)$ the distribution $\Pr(\cdot)$ updated after this labelling. 
%
\item Let $T_{\lambda}$ be the subtree of $T$ made up of all unlabeled nodes of $T$, together with all nodes whose label has been set to $1$ that are children of unlabeled nodes.
\nwdp\ operates on $T_{\lambda}$ as \wdp, with the following difference: Whenever a label $y(i)$ is requested, \nwdp\ determines its value by a majority vote over $\Theta\left(\frac{\log(n/\delta)}{(1-2\lambda)^2}\right)$-many queries selected uniformly at random from $\scL(i)$.
\end{enumerate}

\newcommand{\ktl}{\widetilde{K}(T_{\lambda},{\widehat c})}

\noindent{\bf Proof sketch of Theorem \ref{t:nwdp}}\\
\begin{proof}
Let $\Lambda$ be the set of pairs of leaves whose label has been corrupted by noise. 
%$\frac{\lambda}{\bnm{n}{2}}$-many pairs of leaves of $T$ from $\bnm{L}{2}$. 
A standard Chernoff bound implies that for any fixed subset of $L\times L$ containing at least $\alpha\frac{\log(1/\delta)}{(1-2\lambda)^2}$ pairs (for a suitable constant $\alpha>0$), the probability that the majority of them belongs to $\Lambda$ is at most $\delta$. Let us set for brevity $f(n, \lambda, \delta) = \alpha\frac{\log(n/\delta)}{(1-2\lambda)^2}$. A union bound over the at most $n-1$ internal nodes of $V$ guarantees that for all queries $y(i)$ made by \nwdp\, operating on $T_{\lambda}$ the majority vote over $f(n, \lambda, \delta)$-many queries on pairs of leaves of $\scL(i)$ will produce the correct label (i.e., before noise) of that node with probability at least $1-\delta$. 

Moreover, since the cut ${\widehat c}$ found by \nwdp\ on $T_\lambda$ can be obtained with probability at least  $1-\delta$ from $c^*$ by merging zero or more clusters on $T$, it is immediate to see that $\ktl\le\kt$. 
It is also easy to verify that this inequality holds even in expectation over the prior distributions of cut $c^*$ on $T$ and ${\widehat c}$ on $T_{\lambda}$, that is, $\E_{\Pr_\lambda}\ktl\le\E_{\Pr}\kt$, where $\Pr_\lambda$ denotes the modified prior on tree $T_{\lambda}$ produced after \nwdp's initial preprocessing (Step 1 in the main body of the paper).

Recall that, with probability $\geq 1-\delta$, the behavior of $\nwdp$ on $T$ with prior $\Pr(\cdot)$ is the same as that of $\wdp$ on $T_{\lambda}$ with the updated prior $\Pr_{\lambda}(\cdot)$. Then we can use Lemma~\ref{l:WDPnumberPaths} by replacing $c^*$ with ${\widehat c}$ to claim that the number of paths selected by \nwdp\ before stopping is $\ktl$, and then Lemma~\ref{l:piQueries} to conclude that the expected ({\em w.r.t. $\Pr(\cdot)$}) number of queries made by \nwdp\ is upper bounded with probability $1-\delta$ (over the noise in the labels) by 
\[
\scO\left(f(n, \lambda, \delta)\,\E_{\Pr_{\lambda}}\ktl\log(h(T_{\lambda}))\right)= \scO\left(\frac{\log(n/\delta)}{(1-2\lambda)^2}\,\E_{\Pr}\kt\log h\right)~.
\]
We conclude the proof by showing that with probability at least $1-\delta$ we have 
$d_H(\Sigma^*,\hCC)= \scO\left(\frac{n(\log(n/\delta))^{3/2}}{(1-2\lambda)^3}\right)$.
Since all labels requested by \nwdp\ are simultaneously correct with probability at least $1-\delta$, the distance $d_H(\Sigma^*,\hCC)$ is upper bounded with the same probability by 
$\sum_{i\in L(T_{\lambda})} |L(i)|^2$. For each tree $T_{\lambda}$ constructed by \nwdp\,, and any $i\in L(T_{\lambda})$, we have 
\[
\scO(f(n, \lambda, \delta)) = |L(i)|= \Omega\left(f(n, \lambda, \delta)^{1/2}\right)~.
\] 
%which implies that $\bnm{|L(v)|}{2}=\scO\left((f(n, \lambda, \delta))^2\right)$ and $\bnm{|L(v)|}{2}=\Omega\left(f(n, \lambda, \delta)\right)$. 
Hence, the maximum number of leaves of $T_{\lambda}$ is $\scO\left(\frac{n}{(f(n, \lambda, \delta))^{1/2}}\right)$, and the quantity $\sum_{i\in L(T_{\lambda})} |L(i)|^2$, contributing to $d_H(\Sigma^*,\hCC)$ is upper bounded by
\[
\scO(n\left(f(n, \lambda, \delta))^{3/2}\right) = \scO\left(\frac{n(\log(n/\delta))^{3/2}}{(1-2\lambda)^3}\right)~, 
\]
as claimed.
\end{proof}

%These bounds can  also be improved taking into account the topology of the input tree $T$. For example, the total number of ER mistakes incurred w.h.p. by \nwdp, equal to $\sctO\left(\frac{K}{(1-2\lambda)^{2}}\right)$ in the above theorem statement, can be replaced by $\sctO\left((1-2\lambda)^{-2}\right)$ when $T$ is a caterpillar tree. We could also refine these bounds for what concerns the dependence on $\lambda$ in the ER mistakes formula. 

\section{Missing material from Section \ref{s:nonrealizable}}\label{as:nonrealizable}

\newcommand{\disagr}{\mathrm{\textsl{DIS}}}

\subsection{Proof sketch of Theorem \ref{t:nonrealizable}}\label{as:proofs_nonrealizable}

\begin{proof}
The proof follows from Theorem 2 and 3 in \cite{b+10}, together with the following observations.
\begin{enumerate}
\item For any tree $T$ with $n$ leaves, we have $|\CCC(T,K)| = \scO(n^K)$.
\item When $\DD$ is uniform, the disagreement coefficient $\theta = \theta(\CCC(T,K),\DD)$ is $\scO(K)$. To show this statement, consider the following. For any $c^* \in \CCC(T,K)$ and $r > 0$, let
\begin{align*}
\disagr(c^*,r) =\Bigl\{ (x_1,x_2) \in L\times L\,:\, \exists c' \in \CCC(T,K)\,:\,& \sigma_{\CC(c')} (x_1,x_2) \neq \sigma_{\CC(c^*)} (x_1,x_2)\\ 
										  &\wedge d_H(\Sigma_{\CC(c')},\Sigma_{\CC(c^*)}) \leq r   \Bigl\}~.
\end{align*}
Then in our case $\theta$ is defined as
\[
\theta = \sup_{r > 0} \frac{|\{ (x_1,x_2)  \in \disagr(c^*,r) \}|}{r\,n^2}~.
\]
Now, for any budget $r$ in $\disagr(c^*,r)$, and any $c^* \in \CCC(T,K)$, the number of times we can replicate the perturbation of $c^*$ so as to obtain $c'$ satisfying $d_H(\Sigma_{\CC(c')},\Sigma_{\CC(c^*)}) \leq r$ is at most $K$. This is because any such perturbation will involve a different cluster of $\CC(c^*)$, and therefore disjoint sets of leaves. Moreover, each such perturbation covers $rn^2$ leaves. The worst case that makes $\theta = K$ is when $T$ is a full binary tree, and $\CC(c^*)$ has equally-sized clusters. In all other cases $\theta \leq K$.

\item Regarding the expected running time per round, we give the pseudocode (see Algorithm \ref{a:nonrealizable} in this appendix) of a sequential algorithm, which operates as follows.
In a preliminary phase the input tree $T$ is preprocessed in order to be able to find in {\em constant time} at any time $t$ {\bf (i)} the leftmost and rightmost descendent leaf of any internal node of $T$, and {\bf (ii)} the lowest common ancestor of any two given leaves.\footnote
{
Note that $a_t$ can always be found in constant time after a $\Theta(n)$ time preprocessing phase of $T$ -- see \cite{dt84}.
}
At each time $t$, it receives $\langle(x_{i_t},x_{j_t}),\sigma_t, w_t \rangle$, for some weight $w_t \geq 0$, and label $\sigma_t \in \{-1,+1\}$, and outputs $\err_{t}(\CC({\hat c_{t+1}}))$, based on the past computation of $\CC({\hat c_{t}})$ and $\err_{t-1}(\CC({\hat c_{t}}))$. This can be directly used to compute at each round $\err_{t-1}(\CC({\hat c_{t}}))$ needed by the algorithm, but also the perturbed cut ${\hat c'_t}$ and its associated empirical error $\err_{t-1}(\CC({\hat c'_{t}}))$, once we repeat the computation by perturbing the last item $\langle(x_{i_t},x_{j_t}),\sigma_t, w_t \rangle$ in the training set as follows: $\sigma_t = -\sigma_{\CC({\hat c_t})}(x_{i_t},x_{j_t})$, and $w_t = \infty$. In turn, the above can be used to compute $d_t = \err_{t-1}(\CC({\hat c'_{t}})) - \err_{t-1}(\CC({\hat c_{t}}))$ and probability $p_t$.

The cornerstone of this procedure is to maintain updated over time for each internal node $v$ of $T$ a record storing eight values: 
\begin{itemize}
\item {\em 1st, 2nd, 3rd and 4th values}: positive and negative inter-cluster total weight of all leaves in $L(\lef(v))$ and $L(\ri(v))$; 
\item {\em 5th and 6th values}: positive and negative inter-cluster sum of weights $w(x_i,x_j)$ for all $x_i\in L(\lef(v))$ and all $x_j\in L(\lef(v))$, and 
\item {\em 7th and 8th values}: total intra-cluster negative weight of all the clusters of leaves in $L(\lef(v))$ and $L(\ri(v))$. 
\end{itemize}
When this procedure receives in input triplet $\langle(x_{i_t},x_{j_t}),\sigma_t, w_t \rangle$, it finds $a_t = \lca(x_{i_t},x_{j_t})$. Then the eight records associated with each node on the bottom-up path $\pi(a_t, r)$ are updated in a bottom-up fashion according to the input, whenever necessary. This requires a constant time per node in $V(\pi(a_t, r))$. Finally, $\err_{t}(\CC({\hat c_{t+1}}))$ is obtained by simply summing the total intra-cluster negative weight of all clusters of leaves in $L(\lef(r))$ and $L(\ri(r))$ to 
the total inter-cluster positive weight of all leaves in $L(\lef(r))$ and $L(\ri(v))$, plus the inter-cluster sum of positive weights of the pairs $w(x_i,x_j)$ for all $x_i\in L(\lef(r))$ and $x_j\in L(\lef(r))$. In the special case where the updated clustering is made up of a single cluster containing all leaves of $T$, the procedure outputs the sum of all negative values in the record associated with $r$. In any event, computing this sum requires constant time.

Hence the total time required for performing all operations required at any time $t$ is simply $\scO(|V(\pi(a_t, r)|)$.
%One of the key features that makes this procedure useful and versatile is that it can be invoked at each round $t$ to compute $\err_{t-1}(\CC({\hat c'_{t}}))$ by simply using the input triplet $\langle(x_{i_t},x_{j_t}),-\sigma_{\CC({\hat c'_{t}})} (x_{i_t},x_{j_t}), \infty \rangle$. 
\end{enumerate}
This concludes the proof.
\end{proof}

Finally, in order to compute the clustering at the end of the training phase, it suffices to perform a breadth-first visit of $T$ to find all leaves of $T'_{c^*}$. This requires a time linear in the number of clusters of the clustering found by the algorithm. Then the algorithm outputs the indices of the leftmost and rightmost descendant of each leaf of $T'_{c^*}$, which requires $\Theta(1)$ time per cluster. The total time for giving in output the computed clustering is therefore linear in the number of its own clusters.

\subsection{Pseudocode of the \nr\ algorithm in the non-realizable setting}

\newcommand{\lcost}{\mathrm{left\_cost}}
\newcommand{\rcost}{\mathrm{right\_cost}}
\newcommand{\mcost}{\mathrm{middle\_cost}}
\newcommand{\iscl}{\mathrm{is\_cluster}}
\newcommand{\rig}{\mathrm{right}}
\newcommand{\lchild}{\lef}
\newcommand{\rchild}{\ri}
\newcommand{\lw}{\mathrm{left\_weight}}
\newcommand{\cw}{\mathrm{middle\_weight}}
\newcommand{\rw}{\mathrm{right\_weight}}
\newcommand{\dir}{\mathrm{dir}}
\newcommand{\ww}{\mathrm{weight}}
\newcommand{\cost}{\mathrm{cost}}
\newcommand{\md}{\mathrm{middle}}
\newcommand{\totw}{\mathrm{total\_abs\_weight}}
\newcommand{\totmc}{\mathrm{total\_modified\_cost}}
\newcommand{\caaw}{\mathrm{cost\_after\_adding\_weight}}
\newcommand{\addw}{\texttt{add\_weight}}
\newcommand{\samec}{\mathrm{same\_cluster}}
\newcommand{\curc}{\mathrm{current\_tot\_cost}}

Each internal node of $T$ is associated with a record containing eight values that are maintained updated over time. We start by providing the semantics of these eight values:

	 \vspace{-0.2cm}
    \begin{itemize}	
	\item $\ww(v,\lef,-1)$ and $\ww(v,\lef,+1)$: negative and positive inter-cluster total weight of leaves in $L(\lchild(v))$.
	\item $\ww(v,\md,-1)$ and $\ww(v,\md,+1)$: negative and positive inter-cluster sum of weights $w(\ell_l,\ell_r)$, where $\ell_l\in L(\lchild(v))$ and $\ell_r\in L(\rchild(v))$, respectively. 
	\item $\ww(v,\rig,-1)$ and $\ww(v,\rig,+1)$: negative and positive inter-cluster total weight of leaves in $L(\rchild(v)).$
	\item  $\cost(v,\lef)$ and $\cost(v,\rig)$: intra-cluster total negative weight of clusters of leaves in $L(\lchild(v))$ and $L(\rchild(v))$, respectively.
	\end{itemize} 	
	 %\vspace{0.2cm}	
		Finally, for any internal node $v$ of $T$, we denote by $s(v)$ the following sum:\\
			 \vspace{-0.1cm}
\begin{align*}
	s(v)&\defeq\ww(v,\lef,-1)+\ww(v,\lef,+1)+\ww(v,\md,-1)\\
	&\qquad +\ww(v,\md,+1)+\ww(v,\rig,-1)+\ww(v,\rig,+1)~.
\end{align*}

	 \vspace{-0.15cm}
	 \vspace{-0.15cm}
\begin{algorithm}[!h]
    \SetKwInOut{Input}{\scriptsize{$\triangleright$ INPUT}}
    \SetKwInOut{Output}{\scriptsize{$\triangleright$ OUTPUT}}
    \Input{Sequence of pairs of labeled leaves of the form $\langle (\ell,\ell'),\sigma(\ell,\ell')\rangle$ %(with $\ell\neq\ell'$)
          }
    \Output{Clustering $\CC$ with minimum cost over all clusterings realized by $T$.% given the input online sequence of pairs of leaves (and the corresponding weights set for them).
           }
    \vspace{0.2cm}
    {\bf Init:}\\ 
    \begin{itemize}
    \item {\bf for} $v\in V$ {\bf do} \qquad{\bf if} $v\in L$ $\iscl(v)\gets 1$; {\bf else} $\iscl(v)\gets 0$;
 	\item $\curc\gets 0$;\\
%    \item $\totw\gets 0$;~~\comm{the variable $\totw$ has global scope}\\
    \item Preprocess $T$ in a bottom-up fashion and store for each internal node of $T$ the leftmost and rightmost leaf descendant index.~~\comm{Necessary to output $\CC$ in linear time}
    \item Preprocess $T$ to find the lowest common ancestor of any pair of leaves in constant time.
	\end{itemize}
	       \vspace{-0.2cm}
    %\begin{itemize}
    %\item $\CC\gets \emptyset$;~~\comm{$\CC$ contains all the clusters of $\CC^*$ found so far}\\
    %\end{itemize}
       %\vspace{0.2cm}
%     For any given {\em internal} node $v\in V$, $\lef(v)$ and $\rig(v)$ denote the left-child and right-child of $v$ in $T$ respectively.\\ 
		
	%%%%%%%%%%%%%%%%%%%%%%%%%%%%%%%%%%%%%%%%%%%%%%%%%%%%%%%%
	\vspace{0.2cm}
	\For{$t=1$ \KwTo $\ldots$}{ 
	Receive pair of leaves $(\ell, \ell')$;\\$w(\ell,\ell')\gets 0$;~~\comm{initialize $w(\ell,\ell')$}\\
	$a\gets$ lowest common ancestor of $\ell$ and $\ell'$;~~\comm{we assume $\ell\neq\ell'$}\\
		\comm{save all records for the rollback that will be done later}\\
	$\scS\gets$ list of saved records (eight values per node) of the path $\pi(a,r)$;\\

		\vspace{0.2cm}
	\comm{---------- verify whether $\ell$ and $\ell'$ are in the same cluster of the current optimal clustering ----------}\\
		%\comm{initially we assume $\ell$ and $\ell'$ in different clusters (according to the current optimal clustering)}\\
	%$\samec(\ell,\ell')\gets 0;$\\
	\While{$a\neq r \wedge \iscl(a)=0$}{$a\gets\pa(a)$;}
	{\bf if } $\iscl(a)=1$ {\bf then }$\samec(\ell,\ell')\gets 1;$ {\bf else } $\samec(\ell,\ell')\gets 0;$\\
		\vspace{0.2cm}	
	\comm{---------- compute optimal cost under constraint ----------}\\	
	\eIf{$\samec(\ell,\ell')=1$}{
	\comm{compute the optimal cost of the current clustering constrained by the assumption that $\ell$ and $\ell'$ are in different clusters; $-\infty$ is simulated using a very large negative number}\\	
	$\totmc\gets$\addw$(\ell, \ell', -\infty)$;\\
	}
	{
	\comm{compute the optimal cost of the current clustering constrained by the assumption that $\ell$ and $\ell'$ are in the same cluster; $+\infty$ is simulated using a very large positive number}\\	
	$\totmc\gets$\addw$(\ell, \ell', +\infty)$;\\
	}
	\comm{rollback of the clustering preceding the add of weight $-/+\infty$}\\
	Restore all records of $\scS$;\\
		\vspace{0.2cm}	
	\comm{---------- add weight $w(\ell,\ell')$ if necessary ----------}\\			
	Set:
        \begin{itemize}
        \item Difference $d_t \gets \frac{1}{t-1}\,\left(\totmc - \curc\right)$ ;
        \item Probability $p_t$ as a function of $d_t$ as in Eq. (\ref{e:pt});
	\item $w(\ell,\ell') \gets \frac{\sigma(\ell,\ell')}{p_t}$;
	\item With probability $p_t$, $\curc\gets$\addw$(\ell, \ell', w(\ell,\ell'))$;
        \end{itemize} 
	}
	\vspace{0.2cm}
	\comm{---------- find the current optimal clustering/partition of $L$ ----------}\\
    Perform a breadth-first search on $T$, starting from its root $r$, to create the set $V'$ formed by all nodes $v\in V$ such that
    $\iscl(v)=1$ and for all ancestors $a$ of $v$ we have $\iscl(a)=0$;\\
       \vspace{0.2cm}
	$\CC\gets \emptyset$;\\
    \For{$v \in V'$}{
    	$\CC\gets \CC\cup \{L(v)\}$;\\
    }
	\Return{$\CC$~.}
    \caption{Sequential algorithm for the non-realizable case (\nr).\label{a:nonrealizable}}
\end{algorithm}
	%%%%%%%%%%%%%%%%%%%%%%%%%%%%%%%%%%%%%%%%%%%%%%%%%%%%%%%%
	%%%%%%%%%%%%%%%%%%%%%%%%%%%%%%%%%%%%%%%%%%%%%%%%%%%%%%%%
	%%%%%%%%%%%%%%%%%%%%%%%%%%%%%%%%%%%%%%%%%%%%%%%%%%%%%%%%
	%%%%%%%%%%%%%%%%%%%%%%%%%%%%%%%%%%%%%%%%%%%%%%%%%%%%%%%%
	%%%%%%%%%%%%%%%%%%%%%%%%%%%%%%%%%%%%%%%%%%%%%%%%%%%%%%%%
	%%%%%%%%%%%%%%%%%%%%%%%%%%%%%%%%%%%%%%%%%%%%%%%%%%%%%%%%
	%%%%%%%%%%%%%%%%%%%%%%%%%%%%%%%%%%%%%%%%%%%%%%%%%%%%%%%%
	%%%%%%%%%%%%%%%%%%%%%%%%%%%%%%%%%%%%%%%%%%%%%%%%%%%%%%%%
	%%%%%%%%%%%%%%%%%%%%%%%%%%%%%%%%%%%%%%%%%%%%%%%%%%%%%%%%
	%%%%%%%%%%%%%%%%%%%%%%%%%%%%%%%%%%%%%%%%%%%%%%%%%%%%%%%%
	%%%%%%%%%%%%%%%%%%%%%%%%%%%%%%%%%%%%%%%%%%%%%%%%%%%%%%%%
	%%%%%%%%%%%%%%%%%%%%%%%%%%%%%%%%%%%%%%%%%%%%%%%%%%%%%%%%
	%%%%%%%%%%%%%%%%%%%%%%%%%%%%%%%%%%%%%%%%%%%%%%%%%%%%%%%%
	%%%%%%%%%%%%%%%%%%%%%%%%%%%%%%%%%%%%%%%%%%%%%%%%%%%%%%%%

\newpage

\begin{procedure}[!t]
    \SetKwInOut{Input}{\scriptsize{$\triangleright$ INPUT}}
    \SetKwInOut{Output}{\scriptsize{$\triangleright$ OUTPUT}}
    \Input{Pair of leaves $\ell,\ell'\in L$ (with $\ell\neq\ell'$) and weight $w(\ell,\ell')$}
    \Output{Total clustering cost after adding weight $w(\ell,\ell')$
    %and total sum of the absolute values of all the weights $w(\ell,\ell')$ for $\ell,\ell'\in L$ received so far
}
    %%5Clustering $\CC$ with minimum cost over all clusterings consistent with the cuts of $T$ given an online sequence of $N$-many weights of pairs of leaves.
    %%5{\bf Init:}\\ 
    %%5\For{$\ell\in L$}{
	%%5		$\iscl(\ell)\gets 1$;\\
	%%5}
    %\begin{itemize}
    %\item $\CC\gets \emptyset$;~~\comm{$\CC$ contains all the clusters of $\CC^*$ found so far}\\
    %\end{itemize}
       \vspace{0.2cm}
%     For any given {\em internal} node $v\in V$, $\lef(v)$ and $\rig(v)$ denote the left-child and right-child of $v$ in $T$ respectively.\\ 
		
	%%%%%%%%%%%%%%%%%%%%%%%%%%%%%%%%%%%%%%%%%%%%%%%%%%%%%%%%
	\vspace{0.2cm}
	%%5\For{$t=1$ \KwTo $N$}{ 
	%%5Receive pair of leaves $(\ell, \ell')$ and weight $w(\ell, \ell')$;
	%%5~~\comm{we assume $\ell\neq\ell'$}\\
    $a\gets$ lowest common ancestor of $\ell$ and $\ell'$;\\
	\vspace{0.2cm}
	\comm{ update middle weight record of node $a$ }\\	
	$\ww(a,\md,\sgn(w(\ell,\ell'))\gets \ww(a,\md,\sgn(w(\ell,\ell'))+w(\ell, \ell')$;\\

	\vspace{0.2cm}
	\comm{ set cluster flag of node $a$ }\\		
	\eIf{$s(a)\ge 0$}{$\iscl(a)\gets 1$;\\}{$\iscl(a)\gets 0$;\\}
	
	\vspace{0.2cm}
	\comm{---------- record update of all $a$'s ancestors ----------}\\		
	\While{$a\neq r$}{
	
	\vspace{0.2cm}
	\comm{ set variable $\dir$ to left or right direction from $\pa(a)$ to $a$ }\\	
	\eIf{$a=\lchild(\pa(a))$}{$\dir\gets\lef$;\\}{$\dir\gets\rig$;\\}
	
	\vspace{0.2cm}
	\comm{ update positive and negative inter-cluster weights of node $\pa(a)$ }\\
	\For{$\sigma\in\{+1,-1\}$}{ 
	\eIf{$\iscl(a)=0$}
	{$\ww(\pa(a),\dir,\sigma)\gets\ww(a,\lef,\sigma) +\ww(a,\md,\sigma) + \ww(a,\rig,\sigma)$;\\}
	{$\ww(\pa(a),\dir,\sigma)\gets 0$;\\}
	}
	
	\vspace{0.2cm}
	\comm{ update $\pa(a)$'s cost record relative to node $a$ }\\
	\eIf{$\iscl(a)=0$}
	{$\cost(\pa(a),\dir)\gets \cost(a,\lef)+\cost(a,\rig)$;\\}
	{$\cost(\pa(a),\dir)\gets \cost(a,\lef)+|\ww(a,\lef,-1)|+|\ww(a,\md,-1)|+|\ww(a,\rig,-1)|+\cost(a,\rig)$;\\}

	\vspace{0.2cm}
	\comm{ update cluster flag of $\pa(a)$ }\\	
	\eIf{$s(\pa(a))\ge 0$}{$\iscl(\pa(a))\gets 1$;\\}
	{$\iscl(\pa(a))\gets 0$;\\}
	$a\gets\pa(a)$;\\
	%}
}%endFor%%%%%%%%%%%%%%%%%%%%%%%%%%%%%%%%%%%%%%%%%%%%%%%%%%%%%%%%%%%%%
	\vspace{0.2cm}
%    \comm{update total sum of the absolute values of all weights received so far;}\\
      
%	$\totw\gets\totw+|w(\ell,\ell')|$;~~\comm{the variable $\totw$ has global scope }\\
%	\vspace{0.2cm}
	\comm{---------- compute the total cost of the current optimal clustering ----------}\\
	\eIf{$\iscl(r)=0$}
	{$\caaw\gets \cost(r,\lef)+\ww(r,\lef,+1)+\ww(r,\md,+1)+\ww(r,\rig,+1)+\cost(r,\rig)$\\}
	{$\caaw\gets \cost(r,\lef)+|\ww(r,\lef,-1)|+|\ww(r,\md,-1)|+|\ww(r,\rig,-1)|+\cost(r,\rig)$;\\}
	\vspace{0.2cm}
	\Return{$\caaw$~.}
    \caption{Procedure() \texttt{\addw($\ell,\ell',w(\ell,\ell')$)} }
\end{procedure}

\subsection{Missing material from Section \ref{s:exp}}\label{ass:exp}

In Table \ref{t:results} we report the results of our preliminary experiments. Notice that \nwdp, \nr, and \breadthfirst\ are randomized algorithms. Hence, for these three algorithms we give average results and standard deviation across 10 independent runs of each one of them. 
As a reference, consider that the performance of \best\ (see Section \ref{s:exp} in the main body of the paper) on the three datasets is the following: \sing: 8.26\%, \med: 8,51\%, \comp: 8.81\%. Moreover, since in this dataset we have 10 class labels with approximately the same frequency, both a {\em random} clustering and a degenerate clustering having $n=10000$ singletons would roughly give 10\% error. 

In light of the above, notice that on both \sing\ and \med, the robust breadth-first strategy \breadthfirst\ goes completely off trail, in that it tends to produce clusterings with very few clusters. This behavior is due to the presence in the two hierarchies of long paths starting from the root, which is in turn caused by the way the single and the median linkage functions deal with the outliers contained in the MNIST dataset. 

Finally, one should take into account the fact that when training our active learning algorithms we have used the first 500 labels for parameter tuning. Hence, a fair comparison to \erm\ is one that contrasts the test error of \nwdp, \nr, and \breadthfirst\ at a given number of queries $q$ to the test error of \erm\ at $q+500$ queries. From Table \ref{t:results} one can see that, even with this more careful comparison, \nwdp\ outperforms \erm. On the other hand, \nr\ looks similar to \erm\ on \med\ and \comp, and worse than \erm\ on \sing.

\begin{table}[t!]
\begin{small}
%\begin{center}
\hspace{-0.5in}
  \begin{tabular}{l|l|r|r|r|r|r|r|r}
	&\qquad No. of queries	&250 		&500		&1000		&2000		&5000		&10000		&20000		\\
Tree	&Algorithm\qquad	&		&		&		&		&		&		&		\\
\hline
\sing	&\erm			&8.81		&8,78		&8.39		&8.29		&8.29		&8.29		&8.29		\\
	&\nwdp			&8.29$\pm$0.0	&8.28$\pm$0.0	&8.28$\pm$0.0	&8.29$\pm$0.0	& --		& --		& --		\\
	&\nr			&11.0$\pm$2.0	&8.77$\pm$0.0	&8.43$\pm$0.0	&8.31$\pm$0.0	&8.29$\pm$0.0	& --		& --		\\
	&\breadthfirst		&89.0$\pm$0.0	&89.0$\pm$0.0	&88.0$\pm$0.0	&86.0$\pm$2.0	&87.0$\pm$3.0	&72.0$\pm$10.0&67.0$\pm$10.0\\
\hline
\med	&\erm			&10.30		&10.16		&9.36		&8.91		&8.91		&8.69		&8.65		\\
	&\nwdp			&9.41$\pm$0.1	&9.07$\pm$0.1	&8.88$\pm$0.1	&8.92$\pm$0.1	&8.8$\pm$0.1	&8.8$\pm$0.1	&8.7$\pm$0.1	\\
	&\nr			&10.17$\pm$0.0	&9.37$\pm0.0$	&9.0$\pm$0.0	&8.85$\pm$3.0	& --		& --		& --		\\
	&\breadthfirst		&89.4$\pm$0.0	&88.1$\pm$0.0	&87.0$\pm$0.0	&63.1$\pm$0.0	&18.2$\pm$5.0	&18.0$\pm$3.0	&10.9$\pm$1.0	\\

\hline
\comp	&\erm			&10.65		&10.30		&10.04		&9.26		&9.06		&8.99		&8.93		\\
	&\nwdp			&9.52$\pm$0.0	&9.47$\pm$0.0	&9.44$\pm$0.0	&9.43$\pm$0.0	& --		& --		& -- 		\\
	&\nr			&10.1$\pm$0.0	&10.0$\pm$0.0	&10.0$\pm$0.0	&11.4$\pm$0.6	&10.8$\pm$0.5	&9.0$\pm$0.0	&8.9$\pm$0.0	\\
	&\breadthfirst		&13.5$\pm$0.0	&13.5$\pm$0.0	&9.2$\pm$0.0	&9.1$\pm$0.0	&9.0$\pm$0.0	&9.0$\pm$0.0	&8.9$\pm$0.0	
  \end{tabular}
%\end{center}
\end{small}
\vspace{0.05in}
\caption{Test error (in percentage) vs. number of queries for the various algorithms we tested on the hierarchies \sing, \med, and \comp\ originating from the MNIST dataset (see main body of the paper). Standard deviations are also reported. Missing values on \nwdp\ are due to the fact that the algorithm stops before reaching the desired number of labels. Missing values on \nr\ are instead due to the fact that we stopped the algorithm's execution once we obseved no further test error improvement.\label{t:results}
}
%\vspace{-0.3in}
\end{table}


\begin{thebibliography}{10}

\bibitem{arkin+93}
E. Arkin, H. Meijer, J. Mitchell, D. Rappaport, and S. Skiena. 
\newblock{Decision trees for geometric models.} 
\newblock{In {\em Proc. Symposium on Computational Geometry}, pages 369--378, 1993.}

\bibitem{akbd16}
H. Ashtiani, S. Kushagra, and S. Ben-David. 
\newblock{Clustering with same-cluster queries.} 
\newblock{In {\em Proc. 30th NIPS}, 2016.}

\bibitem{abv17}
P. Awasthi, M. F. Balcan, and K. Voevodski. 
\newblock{Local algorithms for interactive clustering. {\em Journal of Machine Learning Research}, 18, 2017.}

\bibitem{bb08}
M. F. Balcan and A. Blum. 
\newblock{Clustering with interactive feedback.} 
\newblock{In {\em Proc. of the 19th International Conference on Algorithmic Learning Theory}, pages 316--328, 2008.}

\bibitem{bdl09}
Alina Beygelzimer, Sanjoy Dasgupta, and John Langford. 
\newblock{Importance weighted active learning.} 
\newblock{In {\em Proc. ICML}, pages 49--56. ACM, 2009.}

\bibitem{b+10}
Alina Beygelzimer, Daniel Hsu, John Langford, and Tong Zhang. 
\newblock{Agnostic active learning without constraints.} 
\newblock{In {\em Proc. 23rd International Conference on Neural Information Processing Systems}, NIPS' 10, pages 199--207, 2010.}

\bibitem{chkk15}
\newblock{Yuxin Chen, S. Hamed Hassani, Amin Karbasi, and Andreas Krause.}
\newblock{Sequential information maximization: When is greedy near-optimal?} 
\newblock{In {\em Proc. 28th Conference on Learning Theory, PMLR 40, pages 338--363}, 2015.}

\bibitem{chk17}
Yuxin Chen, S. Hamed Hassani, and Andreas Krause. 
\newblock{Near-optimal bayesian active learning with correlated and noisy tests.} 
\newblock{In {\em Proc. 20th International Conference on Artificial Intelligence and Statistics}, 2017.}

\bibitem{cal94}
D. Cohn, L. Atlas, and R. Ladner. 
\newblock{Improving generalization with active learning. {\em Machine Learning}, 15:201--221, 1994.}

\bibitem{c+19}
C. Cortes, G. DeSalvo, C. Gentile, M. Mohri, and N. Zhang. 
\newblock{Region-based active learning.} 
\newblock{In {\em Proc. 22nd International Conference on Artificial Intelligence and Statistics}, 2019.}

\bibitem{hd08}
S. Dasgupta and D. Hsu. 
\newblock{Hierarchical sampling for active learning.} 
\newblock{In {\em Proc. of the 25th International Conference on Machine Learning}, 2008.}

\bibitem{da05}
Sanjoy Dasgupta. 
\newblock{Coarse sample complexity bounds for active learning. {\em In Advances in neural information processing systems}, pages 235--242, 2005.}

\bibitem{d+14}
S. Davidson, S. Khanna, T. Milo, and S. Roy. 
\newblock{Top-k and clustering with noisy comparisons.} 
\newblock{{\em ACM Trans. Database Syst.}, 39(4):35:1--35:39, 2014.}

\bibitem{gk17}
Daniel Golovin and Andreas Krause. 
\newblock{Adaptive submodularity: A new approach to active learning and stochastic optimization.} 
\newblock{In {\em arXiv:1003.3967}, 2017.}

\bibitem{gss13}
Alon Gonen, Sivan Sabato, and Shai Shalev-Shwartz. 
\newblock{Efficient active learning of halfspaces: An aggressive approach.} 
\newblock{{\em Journal of Machine Learning Research}, 14:2583--2615, 2013.}

\bibitem{ha07}
S. Hanneke. A bound on the label complexity of agnostic active learning. 
\newblock{In {\em Proc. 24th International Conference on Machine Learning}, pages 353--360, 2007.}

\bibitem{ha14}
S. Hanneke. Theory of disagreement-based active learning. 
\newblock{{\em Foundations and Trends in Machine Learning}, 7(2-3):131--309, 2014.}

\bibitem{dt84}
Dov Harel and Robert E. Tarjan. 
\newblock{Fast algorithms for finding nearest common ancestors.} 
\newblock{{\em SIAM Journal on Computing}, 13(2):338--355, 1984.}

\bibitem{kosaraju+99}
S. Kosaraju, T. Przytycka, and R. Borgstrom. 
\newblock{On an optimal split tree problem.} 
\newblock{In {\em Proc. 6th International Workshop on Algorithms and Data Structures}, pages 157--168, 1999.}

\bibitem{kub15}
S. Kpotufe, R. Urner, and S. Ben-David. 
\newblock{Hierarchical label queries with data-dependent partitions.}
\newblock{In {\em Proc. 28th Conference on Learning Theory}, pages 1176--1189, 2015.}

\bibitem{ms17b}
A. Mazumdar and B. Saha. 
\newblock{Clustering with noisy queries.} 
\newblock{In {\em arXiv:1706.07510v1}, 2017b.}

\bibitem{me11}
M. Meila. 
\newblock{Local equivalences of distances between clusterings?a geometric perspective.}
\newblock{{\em Machine Learning}, 86(3):369--389, 2012.}

\bibitem{ml18}
Stephen Mussmann and Percy Liang. 
\newblock{Generalized binary search for split-neighborly problems. }
\newblock{In {\em Proc. 21st International Conference on Artificial Intelligence and Statistics (AISTATS) 2018}, 2018.}

\bibitem{no11}
Robert D. Nowak. 
The geometry of generalized binary search. 
\newblock{{\em IEEE Transactions on Information Theory}, 57(12):7893--7906, 2011.}

\bibitem{ra71}
W. M. Rand. Objective criteria for the evaluation of clustering methods.
\newblock{{\em Journal of the American Statistical Association}, 66:846--850, 1971.}

\bibitem{td17}
C. Tosh and S. Dasgupta. 
\newblock{Diameter-based active learning.} 
\newblock{In {\em Thirty-fourth International Conference on Machine Learning (ICML)}, 2017.}

\end{thebibliography}
\end{document}